\PassOptionsToPackage{table,dvipsnames,svgnames,x11names}{xcolor}
\documentclass{article} 
\usepackage{iclr2027_conference,times}
\iclrfinalcopy

\usepackage{amsmath,amsfonts,bm}

\def\eqref#1{equation~\ref{#1}}

\def\1{\bm{1}}

\DeclareMathAlphabet{\mathsfit}{\encodingdefault}{\sfdefault}{m}{sl}
\SetMathAlphabet{\mathsfit}{bold}{\encodingdefault}{\sfdefault}{bx}{n}

\DeclareMathOperator*{\argmax}{arg\,max}
\DeclareMathOperator*{\argmin}{arg\,min}

\usepackage{hyperref}
\usepackage{url}

\usepackage[utf8]{inputenc} 
\usepackage[T1]{fontenc}    
\usepackage{hyperref}       
\usepackage{url}            
\usepackage{booktabs}       
\usepackage{amsfonts}       
\usepackage{nicefrac}       
\usepackage{microtype}      
\usepackage{xcolor}         
\usepackage{siunitx}

\usepackage{multirow}
\usepackage[table]{xcolor}
\usepackage{colortbl}
\usepackage{wrapfig}

\usepackage{hyperref}
\hypersetup{
    colorlinks=true,
    citecolor=blue, 
    urlcolor=blue   
}

\usepackage{amsmath}
\usepackage{amssymb}
\usepackage{mathtools}
\usepackage{amsthm}
\usepackage{algorithm}
\usepackage{algpseudocode}

\usepackage{graphicx}
\usepackage{subcaption}
\usepackage{caption}

\usepackage{xcolor}  
\usepackage[dvipsnames, svgnames, x11names]{xcolor} 
\usepackage{float}

\usepackage[utf8]{inputenc} 
\usepackage[T1]{fontenc}    
\usepackage{hyperref}       
\usepackage{cleveref}
\usepackage{url}            
\usepackage{amsfonts}       
\usepackage{nicefrac}       
\usepackage{microtype}      
\usepackage{xcolor}         

\usepackage{float}
\usepackage{algorithm}
\usepackage{algpseudocode}

\floatstyle{ruled}
\restylefloat{algorithm}

\usepackage{float}
\usepackage{wrapfig}
\usepackage{algorithm}
\usepackage{algpseudocode}

\theoremstyle{plain}
\newtheorem{theorem}{Theorem}[section]
\newtheorem{proposition}[theorem]{Proposition}

\theoremstyle{definition}

\theoremstyle{remark}

\usepackage{float}

\newenvironment{talign*}
{\let\displaystyle\textstyle\csname align*\endcsname}
{\endalign}
\newenvironment{talign}
{\let\displaystyle\textstyle\align}
{\endalign}

\let\oldthanks\thanks
\renewcommand{\thanks}[1]{%
    \bgroup
    \hypersetup{pdfborder={0 0 0}, linkbordercolor=white}%
    \oldthanks{#1}%
    \egroup%
}

\usepackage[most]{tcolorbox}

\newtcolorbox{insightbox}[1][Insight]{
    enhanced,
    breakable,
    colback=gray!5,
    colframe=gray!60,
    title=\textbf{#1},
    fonttitle=\bfseries,
    arc=1.5mm,
    boxrule=0.6pt,
    left=1.2mm,
    right=1.2mm,
    top=1mm,
    bottom=1mm
}

\newtcolorbox{examplebox}[1][Example]{
    enhanced,
    breakable,
    colback=blue!3,
    colframe=blue!45!black,
    title=\textbf{#1},
    fonttitle=\bfseries,
    arc=1.5mm,
    boxrule=0.6pt,
    left=1.2mm,
    right=1.2mm,
    top=1mm,
    bottom=1mm
}

\usepackage{array}
\usepackage{tabularx}
\usepackage[table]{xcolor}
\usepackage[most]{tcolorbox}

\definecolor{headergray}{HTML}{4A4A4A}
\definecolor{questiongray}{HTML}{F1F1F1}
\definecolor{goodrow}{HTML}{DFF9E5}
\definecolor{goodborder}{HTML}{20C95A}
\definecolor{goodtext}{HTML}{009A22}
\definecolor{badtext}{HTML}{D00000}
\definecolor{warntext}{HTML}{9A5A00}

\newtcolorbox{caseframe}{
  enhanced,
  width=0.98\linewidth,
  colback=white,
  colframe=black!45,
  boxrule=0.55pt,
  arc=3pt,
  left=7pt,
  right=7pt,
  top=7pt,
  bottom=7pt,
  boxsep=0pt,
  before skip=5pt,
  after skip=5pt
}

\newtcolorbox{questionbox}{
  enhanced,
  width=\linewidth,
  colback=questiongray,
  colframe=questiongray,
  boxrule=0pt,
  arc=0pt,
  left=6pt,
  right=6pt,
  top=5pt,
  bottom=5pt,
  boxsep=0pt
}

\newtcolorbox{analysisbox}{
  enhanced,
  width=\linewidth,
  colback=green!3,
  colframe=goodborder,
  boxrule=0.65pt,
  arc=3pt,
  left=7pt,
  right=7pt,
  top=9pt,
  bottom=7pt,
  title=\textbf{Analysis},
  coltitle=goodtext,
  colbacktitle=green!12,
  fonttitle=\bfseries,
  boxed title style={
    colframe=green!12,
    boxrule=0pt,
    arc=0pt,
    left=6pt,
    right=6pt,
    top=2pt,
    bottom=2pt
  },
  attach boxed title to top left={xshift=7pt,yshift=-2pt}
}

\title{SR-OPSD: Self-Referenced On-Policy Self-Distillation}

\usepackage{amsmath,amssymb}
\usepackage{array}
\usepackage{tabularx}

\usepackage[table]{xcolor}
\usepackage[most]{tcolorbox}

\usepackage{microtype}

\definecolor{headergray}{HTML}{4A4A4A}
\definecolor{questiongray}{HTML}{F1F1F1}
\definecolor{goodrow}{HTML}{DFF9E5}
\definecolor{goodborder}{HTML}{20C95A}
\definecolor{goodtext}{HTML}{009A22}
\definecolor{badtext}{HTML}{D00000}
\definecolor{warntext}{HTML}{9A5A00}

\definecolor{framegray}{HTML}{444444}
\definecolor{textgray}{HTML}{666666}
\definecolor{goodgreen}{HTML}{00B950}
\definecolor{badred}{HTML}{FF0050}

\newtcolorbox{outerframe}{
  enhanced,
  width=0.98\linewidth,
  colback=white,
  colframe=framegray,
  boxrule=0.75pt,
  arc=0.5pt,
  left=14pt,
  right=14pt,
  top=13pt,
  bottom=13pt,
  boxsep=0pt
}

\input{qualitative_style.tex}

\author{\textbf{Zhuo Sun}$^{1,2}$\thanks{Equal Contribution.} ~\thanks{Corresponding Author(s). Correspondence to \url{zhuosunreid@outlook.com}, ~\url{habhuz@dtu.dk}, ~ \url{zengli@sribd.cn}.} \quad Entong Li$^{3,*}$ \enspace Yanlong Zhao$^{4,3,*}$ \enspace \textbf{Xiaoyuan Cheng}$^{5,*}$,\quad \textbf{Wenxuan Yuan}$^{2}$\\
\textbf{Kaiyu Li}$^{5}$ ~~ \textbf{Che Liu}$^{2}$ ~~ \textbf{Huihang Liu}$^{1}$ ~~ \textbf{Baihua He}$^{4}$, ~~\textbf{Xinyu Zhang}$^{4,8}$\\
\textbf{Harrison Bo Hua Zhu}$^{6,2,7, \dagger}$  ~~ \textbf{Li Zeng}$^{3,\dagger}$\\[0.2cm]
$^{1}$Shanghai University of Finance and Economics, $^{2}$Imperial College London,\\
$^{3}$Shenzhen Research Institute of Big Data, $^{4}$University of Science and Technology of China, \\
$^{5}$University College London,
$^{6}$Technical University of Denmark, \\
$^{7}$University of Copenhagen, $^{8}$AMSS, Chinese Academy of Sciences\\[0.1cm]
}

\begin{document}

\maketitle

\begin{abstract}
On-policy self-distillation (OPSD) converts feedback into dense token-level supervision on student-generated trajectories, complementing reinforcement learning with sparse outcome rewards. Its self-teacher, derived from the student's current or exponentially averaged parameters and conditioned on additional context, evolves alongside the student and its rollout context distribution. The benefit of modifying this moving target depends on how target--student probability mismatches translate into updates. We propose \emph{Self-Referenced On-Policy Self-Distillation (SR-OPSD)}, which constructs a normalized geometric target from the self-teacher and a frozen initial policy, then minimizes the forward R\'enyi divergence from this target to the student. The interpolation coefficient controls the self-teacher's contribution, while the R\'enyi order controls the power weighting of target-to-student probability ratios in the gradient. For fixed contexts and target components, we establish a conditional variational characterization and derive the exact token-logit gradient, revealing how anchoring and projection jointly shape the effective update target. Experiments across scientific reasoning, tool use, mathematical reasoning, and code generation demonstrate strong performance across multiple model families and scales. Ablations further show that reference anchoring can improve or degrade performance depending on the projection objective, supporting the joint design of target construction and projection geometry.
\end{abstract}

\section{Introduction}
Post-training of large language models (LLMs) encompasses supervised instruction tuning,
preference optimization, and reinforcement learning
\citep{wei2022finetuned,sanh2022multitask,
NEURIPS2022_b1efde53,rafailov2023direct, he2026unifying, wu2026fdivergence}.
Among reinforcement-learning methods, RLHF methods
commonly optimize a learned reward model using proximal policy
optimization (PPO)
\citep{schulman2017proximal,NEURIPS2022_b1efde53},
whereas recent reasoning systems increasingly employ reinforcement learning
with verifiable rewards (RLVR), using algorithms such as Group Relative
Policy Optimization (GRPO) \citep{shao2024deepseekmath} and its variants \citep{yu2025dapo}. Methods such as Direct Preference
Optimization (DPO) \citep{rafailov2023direct} optimize preference data without explicit reward modeling. In terms of reasoning, RLVR is particularly attractive because the correctness of the outcome can often be assessed automatically through external checks,
unit tests, or rule-based verifiers. However, these methods assign a scalar-valued reward to an entire response. Although this reward induces
policy-gradient updates over the generated tokens, it does not directly
identify which intermediate tokens or reasoning steps caused success or
failure, resulting in coarse and sparse credit assignment over long
sequences.

On-policy distillation (OPD) offers an alternative approach by utilizing the dense
token-level supervision on trajectories sampled from the policy to be optimized (also known as the student policy) \citep{gu2023minillm,agarwal2023onpolicy, zhu2026manyfaces, jin2026entropyaware}. A teacher policy evaluates
student-generated prefixes, reducing the
state-distribution mismatch associated with offline distillation. Existing on-policy distillation methods can be viewed through three design choices including: the divergence objective, the source of supervision, and the mechanism used to stabilize policy optimization. These choices are often designed largely in isolation, which is relatively
benign in conventional offline distillation where both the teacher and the training policy are fixed. This benefit,
however, generally requires continued access to an extra teacher policy. 
On-policy self-distillation (OPSD) instead uses two views of the same evolving
model: the policy we aim to optimize is conditioned on the original prompt, whereas the so-called self-teacher policy is the same policy conditioned on feedback, privileged information, or richer contextual signals \citep{zhao2026selfdistilledreasoner,hubotter2026reinforcement}.
OPSD can therefore convert feedback, reflections, verified traces, or successful responses into dense token-level supervision.

\begin{figure*}[t]
\centering
\begin{subfigure}[t]{0.31\textwidth}
    \centering
    \includegraphics[width=\linewidth]
    {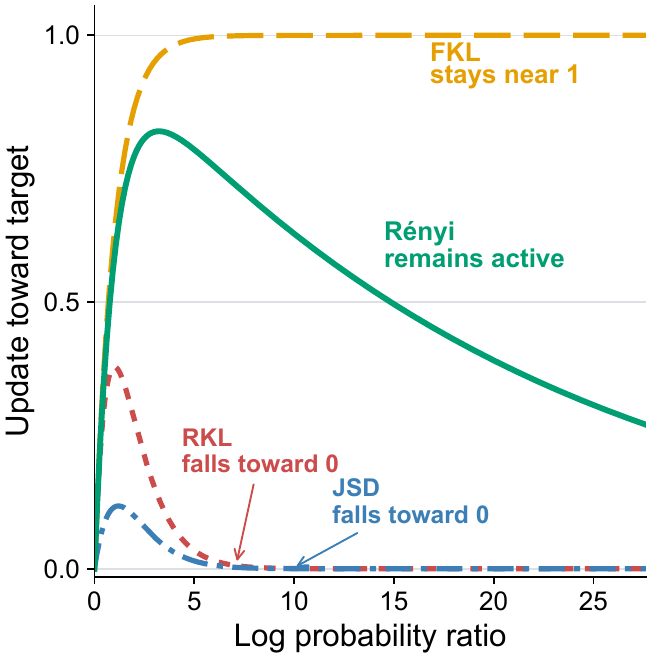}
    \caption{Divergence-dependent token correction.}
    \label{fig:divergence_correction}
\end{subfigure}
\hfill
\begin{subfigure}[t]{0.31\textwidth}
    \centering
    \includegraphics[width=\linewidth]
    {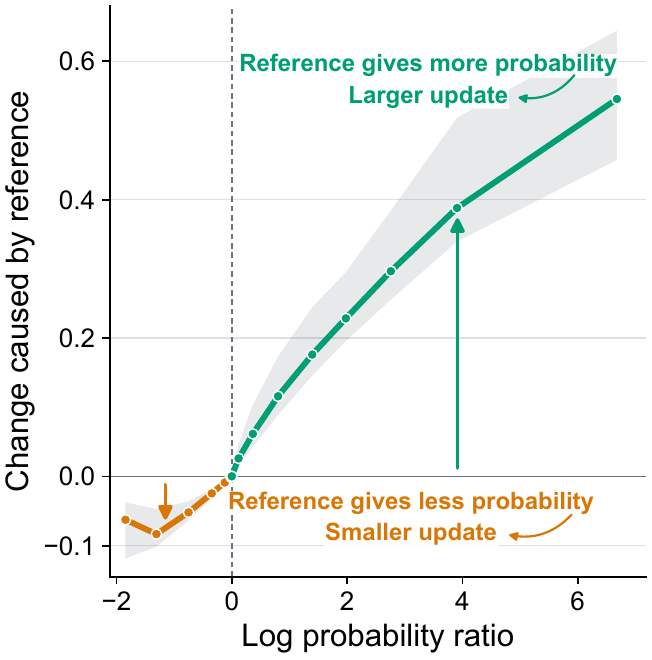}
    \caption{Reference-induced change in correction.}
    \label{fig:self_reference}
\end{subfigure}
\hfill
\begin{subfigure}[t]{0.31\textwidth}
    \centering
    \includegraphics[width=\linewidth]
    {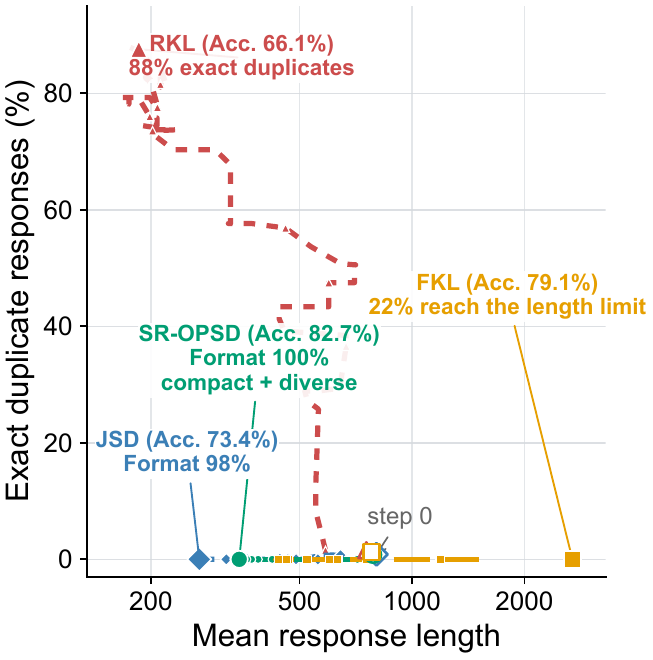}
    \caption{Observed training trajectories.}
    \label{fig:physics_trajectories}
\end{subfigure}
\caption{
Motivation and empirical behavior of SR-OPSD.
$\pi_\theta$, $\pi_{\bar\theta}$, and $\pi_{\mathrm{ref}}$
denote the student, self-teacher, and reference policies.
\textbf{(a)} Binary token-logit correction versus
$\log(\pi_{\bar\theta}(a)/\pi_\theta(a))$,
with $\pi_{\bar\theta}(a)=0.1$.
\textbf{(b)} Change in token-logit correction induced by
reference anchoring,
$\pi_\alpha^\star\propto
\pi_{\bar\theta}^{\alpha}\pi_{\mathrm{ref}}^{1-\alpha}$,
versus $\log(\pi_{\mathrm{ref}}(a)/
\pi_{\bar\theta}(a))$. The ``update'' here is the local correction applied to a token logit returned by $\pi_\theta$.
Points and shading indicate bin medians and
interquartile ranges.
\textbf{(c)} Training trajectories on SciKnowEval Physics,
showing mean response length and exact-response duplication.
SR-OPSD combines concise responses and low duplication with the highest final accuracy and full format validity.
}
\label{fig:sr_opsd_motivation}
\vspace{-12pt}
\end{figure*}

\paragraph{Motivation}
We consider on-policy self-distillation with an evolving feedback-conditioned self-teacher.
The policy to be optimized $\pi_\theta$ generates its own training contexts, while the supervisory target (also known as the self-teacher policy) $\pi_{\bar\theta}$
is updated from it. Both the supervisory target and the rollout context distribution therefore change throughout training. At each visited prefix, target construction determines the probability discrepancies presented to the student, and the divergence determines how these discrepancies translate into updates. This raises a central design question: 
\begin{center}
\emph{How should the supervisory target and the student's response to it be controlled jointly?}    
\end{center}

The distinction between forward KL (FKL)~\citep{zhao2026selfdistilledreasoner} and reverse KL (RKL)~\citep{hubotter2026reinforcement} illustrates the role of the projection objective.
In the binary illustration in \Cref{fig:divergence_correction}, the target assigns a fixed positive probability to a token. As the student's probability of that token approaches zero,
the FKL logit correction approaches a nonzero plateau, whereas the RKL and JSD corrections vanish. Forward R\'enyi divergence provides an order-dependent response: for $0<\rho<1$, its correction decays as a power of the student probability, more slowly than under RKL or JSD in this regime. The order $\rho$ thus controls the response to tokens that the student strongly underweights relative to its target. Reference-policy anchoring provides a complementary control through target construction \citep{yang2026learningbeyondteacher}. Incorporating a reference policy changes the target probabilities and can alter both the direction and magnitude of the resulting correction, as illustrated in \Cref{fig:self_reference}. Because reference anchoring changes the probability ratios presented to the student policy, its effect depends on how the divergence weights those ratios. The runs in \Cref{fig:physics_trajectories} exhibit distinct patterns: RKL produces more duplicate responses, FKL generates longer responses, and SR-OPSD combines higher accuracy with concise outputs and low duplication.

We therefore propose \emph{Self-Referenced On-Policy Self-Distillation
(SR-OPSD)}, which anchors the supervisory target to a frozen copy of the student's own initial policy and matches this target through forward R\'enyi projection. Specifically, using the frozen initial policy as a reference policy $\pi_{\mathrm{ref}}$, SR-OPSD constructs the normalized geometric target $\pi_\alpha^\star \propto \pi_{\bar\theta}^{\alpha}\pi_{\mathrm{ref}}^{1-\alpha}$ and minimizes $D_\rho(\pi_\alpha^\star\|\pi_\theta)$ on student-generated prefixes. The coefficient $\alpha\in[0,1]$ controls \emph{where} the target lies along the geometric interpolation between the self-teacher and reference policies, while
$\rho\in(0,1)$ controls \emph{how} target-to-student probability ratios are weighted in the logit update. Together, they control how reference information and
evolving self-teacher supervision shape student learning.

\vspace{-5pt}
\paragraph{Main Contributions}
Our contributions are threefold. (i) We propose SR-OPSD, which jointly controls target construction and policy updates by coupling a normalized geometric target, formed from the self-teacher and the frozen initial student policy, with forward R\'enyi projection. (ii) At fixed student-generated contexts, we establish a conditional variational characterization and derive the exact token-logit gradient. The gradient identifies the effective update target as a normalized geometric mixture of the self-teacher, reference, and current student, making the roles of $\alpha$ and $\rho$ in the update explicit. (iii) We demonstrate competitive performance across scientific reasoning, tool use, mathematical reasoning, and code generation, covering different feedback sources, model families, and scales.

\section{Background}
\label{sec: background}

\paragraph{Notation}
We denote the policy to be optimized by $\pi_\theta$, the self-teacher policy by
$\pi_{\bar\theta}$, and the reference policy by
$\pi_{\mathrm{ref}}$. Given a prompt $x\sim\mathcal D$ and an on-policy
response $y=(y_1,\ldots,y_L)$, the student prefix at position $t$ is
$h_t=(x,y_{<t})$. We write $h_t^{\mathrm{sf}}$ for the corresponding
self-feedback or privileged context supplied to the self-teacher and
$h_t^{\mathrm{ref}}$ for the context supplied to the reference policy. When a
statement is token-level and the contexts are fixed, we suppress the position
index and write $h$, $h_{\mathrm{sf}}$, and $h_{\mathrm{ref}}$. Thus,
$h$, $h_{\mathrm{sf}}$, and $h_{\mathrm{ref}}$ always denote three possibly
different fixed contexts. In this work,
$h_t^{\mathrm{sf}}=h_t^{\mathrm{ref}}=(x^+,y_{<t})$, where $x^+$ augments $x$
with successful examples or feedback. All token-level distributions are defined over the same next-token vocabulary $\mathcal V$. For the theoretical derivations, all compared distributions are normalized on a common support. The
operator $\operatorname{sg}[\cdot]$ denotes stop-gradient.

\paragraph{On-Policy Self-Distillation}
Given a self-teacher policy $\pi_{\bar\theta}$, on-policy self-distillation trains the
student on its own sampled prefixes, for example through the reverse-KL loss
\begin{talign*}
\mathbb E_{\substack{x\sim\mathcal D\\
y\sim\pi_{\theta_{\mathrm{old}}}(\cdot\mid x)}}
\left[
\sum_{t=1}^{L}
D_{\mathrm{KL}}\!\left(
\pi_\theta(\cdot\mid h_t)
\,\middle\|\,
\operatorname{sg}\!\left[\pi_{\bar\theta}\right]
\right)
\right].   
\end{talign*}
Here $\pi_{\theta_{\mathrm{old}}}$ denotes the rollout policy used to collect
the current on-policy batch; it coincides with the current student before the
optimization step. Such a self-teacher policy is often an exponential-moving-average of $\pi_\theta$ with additional context information. The
additional context $h_t^{\mathrm{sf}}$ can contain textual feedback, verified
solutions, reflections, or successful sibling responses that are unavailable in the ordinary student context $h_t$. This construction also supports forward KL and JSD, evaluated between conditional token distributions at the same visited prefixes.

\paragraph{Self-Teacher Policy and Its Updates}
The self-teacher policy $\pi_{\bar\theta}$ may be a stopped copy, a lagged snapshot, or an
exponential-moving-average version of the policy to be optimized $\pi_{\theta}$. In this work, we follow \citep{hubotter2026reinforcement} and use $\beta\in[0,1)$ as the exponential moving average coefficient and adopt the convention
\(
\bar\theta_{k+1}
=
\beta\bar\theta_k+(1-\beta)\theta_{k+1}.
\)
Thus, larger $\beta$ produces a slower-moving self-teacher policy. During each update, $\bar\theta_k$ and the resulting self-teacher policy are held
fixed.

\paragraph{R\'enyi divergence}
For positive probability distributions $p$ and $q$, the R\'enyi
divergence of order $0<\rho<1$ is
\(
D_\rho(p\|q)
=\frac{1}{\rho-1}
\log\sum_{a\in\mathcal V}p(a)^\rho q(a)^{1-\rho}
\).
We focus on this order range and define
$D_1(p\|q)=D_{\mathrm{KL}}(p\|q)$ by continuity
\citep{van2014renyi}.
When $p$ is the target and $q$ the student, we refer to $D_\rho(p\|q)$ as the forward projection. The two argument orders satisfy $D_\rho(p\|q)
=\frac{\rho}{1-\rho}D_{1-\rho}(q\|p)$. This identity relates the projection direction to the divergence order.

\section{Method}
\label{sec: method}

\subsection{Revisiting On-Policy Self-Distillation: Reference-Regularized On-Policy Self-Distillation}

We now consider the following token-level formulation of the reverse-KL-based on-policy self-distillation objective at fixed contexts:

\begin{talign*}
\label{eq: obj of opsd}
\max_{\theta}\;
\mathbb E_{\substack{x\sim\mathcal D\\
y\sim\pi_{\theta_{\mathrm{old}}}(\cdot\mid x)}}
\left[
\sum_{t=1}^{L}
-
D_{\mathrm{KL}}\!\left(
\pi_\theta(\cdot\mid h_t)
\,\middle\|\,
\operatorname{sg}\!\left[
\pi_{\bar\theta}(\cdot\mid h_t^{\mathrm{sf}})
\right]
\right)
\right].
\end{talign*}
At each token position, let $(h,h_{\mathrm{sf}})$ denote the student, self-teacher contexts. During the inner update, these contexts and the self-teacher policy are held fixed. Under regular conditions, the following objective
\begin{talign}
\mathcal J_\alpha
(\theta;h,h_{\mathrm{sf}},h_{\mathrm{ref}})
:=
\mathbb E_{a\sim\pi_\theta(\cdot\mid h)}
\left[
\alpha\log
\frac{
\operatorname{sg}\!\left[
\pi_{\bar\theta}(a\mid h_{\mathrm{sf}})
\right]
}{
\pi_{\mathrm{ref}}(a\mid h_{\mathrm{ref}})
}
\right]
-
D_{\mathrm{KL}}\!\left(
\pi_\theta(\cdot\mid h)
\,\middle\|\,
\pi_{\mathrm{ref}}(\cdot\mid h_{\mathrm{ref}})
\right)
\end{talign}
satisfies $\mathcal J_\alpha
(\theta;h,h_{\mathrm{sf}},h_{\mathrm{ref}})
=
-
D_{\mathrm{KL}}\!\left(
\pi_\theta(\cdot\mid h)
\,\middle\|\,
\pi_\alpha^\star(\cdot\mid h_{\mathrm{sf}},h_{\mathrm{ref}})
\right)+
\log Z_\alpha(h_{\mathrm{sf}},h_{\mathrm{ref}})$ where $\pi_\alpha^\star(a\mid h_{\mathrm{sf}},h_{\mathrm{ref}}) \propto
\operatorname{sg}\!\left[
\pi_{\bar\theta}(a\mid h_{\mathrm{sf}})
\right]^\alpha
\pi_{\mathrm{ref}}(a\mid h_{\mathrm{ref}})^{1-\alpha}$.  See Appendix~\ref{appendix: proof of variational characterization of opsd} for details of our derivation. Consequently, for fixed contexts and frozen target components, maximizing $\mathcal J_\alpha$ with respect to $\theta$ is equivalent to minimizing the reverse KL from the student to $\pi_\alpha^\star$. When $\alpha=1$, $Z_1=1$ and $\pi_1^\star(\cdot\mid h_{\mathrm{sf}},h_{\mathrm{ref}})
=\operatorname{sg}[\pi_{\bar\theta}(\cdot\mid h_{\mathrm{sf}})]$, recovering the token-level reverse-KL-based on-policy self-distillation objective. For $0<\alpha<1$, this formulation incorporates reference
regularization into on-policy self-distillation through the supervisory target \citep{hubotter2026reinforcement, yang2026learningbeyondteacher}. The geometric target $\pi_\alpha^\star$ interpolates between the self-teacher and reference policies.

\subsection{Self-Referenced On-policy Self-distillation with R\'enyi Projection on Student Rollouts}

\Cref{fig:divergence_correction,fig:self_reference} illustrate how projection choice and reference anchoring shape the token-logit correction. We use the R\'enyi family to parameterize the projection, first expressing the objectives against the self-teacher and then present how the forward objective being connected to the geometric target $\pi_\alpha^\star$.

During the inner update, the sampled contexts and the teacher and reference policies are held fixed. For $0<\rho<1$ and strictly positive token probabilities, the reverse R\'enyi projection maximizes the following objective with respect to $\theta$:
\begin{talign*}
-\sum_t D_\rho\!\left(\pi_\theta(\cdot\mid h_t)\,\middle\|\,\operatorname{sg}\!\left[\pi_{\bar\theta}(\cdot\mid h_t^{\mathrm{sf}})\right]\right) =\sum_t\frac{1}{1-\rho}\log\mathbb E_{a\sim\pi_\theta(\cdot\mid h_t)}[(\frac{\operatorname{sg}\![\pi_{\bar\theta}(a\mid h_t^{\mathrm{sf}})]}{\pi_\theta(a\mid h_t)})^{1-\rho}].
\end{talign*}
The forward R\'enyi projection maximizes the corresponding objective
\begin{talign}
-\sum_t D_\rho\!\left(\operatorname{sg}\!\left[\pi_{\bar\theta}(\cdot\mid h_t^{\mathrm{sf}})\right]\,\middle\|\,\pi_\theta(\cdot\mid h_t)\right) =\sum_t\frac{1}{1-\rho}\log\mathbb E_{a\sim\pi_\theta(\cdot\mid h_t)}[(\frac{\operatorname{sg}\![\pi_{\bar\theta}(a\mid h_t^{\mathrm{sf}})]}{\pi_\theta(a\mid h_t)})^\rho].
\end{talign}
Both objectives admit a student-expectation representation, with respective powers $1-\rho$ and $\rho$ applied to the teacher-to-student probability ratio. We adopt the forward parameterization, in which $\rho$ directly controls this power.

SR-OPSD applies forward R\'enyi projection to the normalized geometric target $\pi_\alpha^\star$, combining reference anchoring through $\alpha$ with projection control through $\rho$. To connect this projection to the reference-regularized formulation above, we construct a R\'enyi-aggregated token functional from the same teacher-to-reference log-ratio. The following proposition shows that combining this functional with a R\'enyi penalty relative to $\pi_{\mathrm{ref}}$ yields an objective equivalent to minimizing $D_\rho(\pi_\alpha^\star\|\pi_\theta)$ at fixed contexts.

\begin{wrapfloat}{algorithm}{R}{0.65\textwidth}
\caption{SR-OPSD}
\label{alg:renyi_sropsd}
\small
\begin{algorithmic}[1]
\Require Initial policy $\pi_{\theta_0}$
\State Initialize the self-teacher:
$\bar\theta_0 \leftarrow \theta_0$
\State Set the frozen reference:
$\pi_{\mathrm{ref}}
\leftarrow \operatorname{sg}[\pi_{\theta_0}]$
\For{$k=0,1,\ldots, K-1$}
    \State Sample prompts $x\sim\mathcal D$ and responses
    $y\sim\pi_{\theta_k}(\cdot\mid x)$
    \For{$t=1,\ldots,|y|$}
        \State Construct $h_t=(x,y_{<t})$ and
        $h_t^{\mathrm{sf}}=h_t^{\mathrm{ref}}=(x^+,y_{<t})$.
        \State Form the frozen target
        $\pi_{\alpha,k}^{\star}
        (\cdot\mid h_t^{\mathrm{sf}},h_t^{\mathrm{ref}})$
    \EndFor
    \State Update the policy
\Statex \hspace{\algorithmicindent}%
\resizebox{\dimexpr\linewidth-\algorithmicindent\relax}{!}{%
$\displaystyle
\theta_{k+1}
\leftarrow \mathrm{Update}~(
\theta_k,~
-\eta\nabla_\theta
\sum_t
D_\rho\!\left(
\pi_{\alpha,k}^{\star}
(\cdot\mid h_t^{\mathrm{sf}},h_t^{\mathrm{ref}})
\,\middle\|\,
\pi_\theta(\cdot\mid h_t)
\right)\bigg|_{\theta=\theta_k})
$%
}
    \State Update the self-teacher:
    $\bar\theta_{k+1}
    \leftarrow
    \beta\bar\theta_k+(1-\beta)\theta_{k+1}$
\EndFor
\State \Return $\pi_{\theta_{K}}$
\end{algorithmic}
\end{wrapfloat}

\begin{proposition}[Token-Level Variational Characterization of SR-OPSD]
\label{prop:opsd_forward_renyi_regularized_equivalence}
Fix $h$, $h_{\mathrm{sf}}$, and $h_{\mathrm{ref}}$. Let
$\alpha\in[0,1]$ and $\rho>0$, $\rho\neq1$, and assume that all token
probabilities are strictly positive. Define the token-level pseudo-reward $r_\alpha(a;h_{\mathrm{sf}},h_{\mathrm{ref}})
:=
\alpha\log
\frac{
\operatorname{sg}\!\left[
\pi_{\bar\theta}(a\mid h_{\mathrm{sf}})
\right]
}{
\pi_{\mathrm{ref}}(a\mid h_{\mathrm{ref}})
}$.
Define the R\'enyi-aggregated token functional $\mathcal R_{\alpha,\rho}
(\theta;h,h_{\mathrm{sf}},h_{\mathrm{ref}}):=
-\frac{1}{\rho-1}
\log
\frac{
\sum_{a\in\mathcal V}
\pi_{\mathrm{ref}}(a\mid h_{\mathrm{ref}})^\rho
\pi_\theta(a\mid h)^{1-\rho}
\exp\!\left(
\rho
r_\alpha(a;h_{\mathrm{sf}},h_{\mathrm{ref}})
\right)
}{
\sum_{a\in\mathcal V}
\pi_\theta(a\mid h)^\rho
\pi_{\mathrm{ref}}(a\mid h_{\mathrm{ref}})^{1-\rho}
}$.
Let $\mathcal J_{\alpha,\rho}
(\theta;h,h_{\mathrm{sf}},h_{\mathrm{ref}}) :=
\mathcal R_{\alpha,\rho}
(\theta;h,h_{\mathrm{sf}},h_{\mathrm{ref}})
-
D_\rho\!\left(
\pi_\theta(\cdot\mid h)
\,\middle\|\,
\pi_{\mathrm{ref}}(\cdot\mid h_{\mathrm{ref}})
\right)$.
Then
\begin{align}
\mathcal J_{\alpha,\rho}
(\theta;h,h_{\mathrm{sf}},h_{\mathrm{ref}}) =
-
D_\rho\!\left(
\pi_\alpha^\star(\cdot\mid h_{\mathrm{sf}},h_{\mathrm{ref}})
\,\middle\|\,
\pi_\theta(\cdot\mid h)
\right)
-
\frac{\rho}{\rho-1}
\log
Z_\alpha(h_{\mathrm{sf}},h_{\mathrm{ref}}).
\end{align}
Therefore, for fixed contexts and frozen target components, maximizing
$\mathcal J_{\alpha,\rho}$ with respect to $\theta$ is equivalent to
minimizing the forward R\'enyi projection from
$\pi_\alpha^\star(\cdot\mid h_{\mathrm{sf}},h_{\mathrm{ref}})$ to
$\pi_\theta(\cdot\mid h)$.
\end{proposition}

See Appendix~\ref{appendix:proof_of_opsd_forward_renyi_regularized_equivalence} for the proof. In this variational formulation, $\mathcal R_{\alpha,\rho}$ depends on the current student policy. When $\alpha=1$, we have $Z_1=1$, and the objective reduces to on-policy self-distillation with forward R\'enyi divergence, with the self-teacher $\pi_{\bar{\theta}}$ as the sole target.

The proposed method, SR-OPSD, utilizes this formulation at each student-generated position $t$, using the geometric target $\pi_\alpha^\star(a\mid h_t^{\mathrm{sf}},h_t^{\mathrm{ref}})\propto\operatorname{sg}[\pi_{\bar\theta}(a\mid h_t^{\mathrm{sf}})]^\alpha\pi_{\mathrm{ref}}(a\mid h_t^{\mathrm{ref}})^{1-\alpha}$. We use a frozen copy of the initial policy as the reference policy, $\pi_{\mathrm{ref}} := \operatorname{sg}[\pi_{\theta_0}]$, where $\theta_0$ denotes the initial parameters. Since the reference comes from the student itself, we refer to this anchoring scheme as \emph{self-referenced}. The corresponding token-level objective is:
\begin{equation}
\begin{aligned}
\max_\theta\;
&\mathbb E_{\substack{
x\sim\mathcal D\\
y\sim\pi_{\theta_{\mathrm{old}}}(\cdot\mid x)
}}
\sum_{t=1}^{L}
\frac{1}{1-\rho}[
\log\sum_{a\in\mathcal V}
\exp(
(1-\rho)\log\pi_\theta(a\mid h_t)
\\
&\qquad\qquad
+\rho\alpha
\log\operatorname{sg}\!\left[
\pi_{\bar\theta}(a\mid h_t^{\mathrm{sf}})
\right]
+\rho(1-\alpha)
\log\pi_{\mathrm{ref}}(a\mid h_t^{\mathrm{ref}})
)].
\end{aligned}
\end{equation}
See \Cref{alg:renyi_sropsd} for the pseudo-code, where
$\mathrm{Update}(\cdot)$ denotes an optimizer such as stochastic gradient
descent or Adam. The algorithm states the normalized,
unclipped objective. See Appendix~\ref{app:implementation_details} for more details. The coefficient $\alpha$ determines the self-referenced target: $\alpha=1$ recovers the self-teacher, whereas $\alpha=0$ recovers the reference policy. The order $\rho$ controls the projection toward this target.

\subsection{How Anchoring and Projection Determine the Update}

To characterize how anchoring and projection jointly shape the student update, we derive the exact token-logit gradient of the SR-OPSD objective. The following proposition expresses this gradient as the difference between the student distribution and an effective target that combines $\pi_\alpha^\star$ with the current student.

\begin{proposition}[Logit Gradient of SR-OPSD]
\label{prop:renyi_logit_gradient}
Fix $h$, $h_{\mathrm{sf}}$, and $h_{\mathrm{ref}}$, and treat
$\pi_\alpha^\star(\cdot\mid h_{\mathrm{sf}},h_{\mathrm{ref}})$
as frozen during the student update. Let
$z_\theta(a\mid h)$ denote the student logit such that $\pi_\theta(a\mid h)=
\frac{\exp(z_\theta(a\mid h))}
{\sum_{b\in\mathcal V}\exp(z_\theta(b\mid h))}$.
For $\rho>0$, $\rho\neq1$, define $\widetilde{\pi}_{\alpha,\rho,\theta}
(a\mid h,h_{\mathrm{sf}},h_{\mathrm{ref}}) \propto  \pi_\theta(a\mid h)^{1-\rho} ~ \pi_\alpha^\star(a\mid h_{\mathrm{sf}},h_{\mathrm{ref}})^\rho$.
Then $\frac{\partial}{\partial z_\theta(a\mid h)}
D_\rho\!\left(
\pi_\alpha^\star(\cdot\mid h_{\mathrm{sf}},h_{\mathrm{ref}})
\,\middle\|\,
\pi_\theta(\cdot\mid h)
\right)
=
\pi_\theta(a\mid h)
-
\widetilde{\pi}_{\alpha,\rho,\theta}
(a\mid h,h_{\mathrm{sf}},h_{\mathrm{ref}})$.
\end{proposition}

For $0<\rho<1$, the effective target $\widetilde{\pi}_{\alpha,\rho,\theta}$ is a normalized geometric mixture of the self-teacher, reference, and current student policies, with exponents $\alpha\rho$, $(1-\alpha)\rho$, and $1-\rho$, respectively. Thus, $\alpha$ sets the balance between self-teacher supervision and reference anchoring, while $\rho$ controls their combined weight relative to the current student. Together, these parameters determine the effective target toward which the token-logit update moves the student. This expression links the two controls in \Cref{fig:divergence_correction,fig:self_reference}: reference anchoring changes the target-to-student probability ratios, and the R\'enyi order determines how these ratios enter the token-logit correction.

\section{Experiments}
\label{sec:experiments}
We evaluate the proposed method across various representative scenarios that capture different feedback regimes and
levels of supervision, including: (i) Science Q\&A and tool use benchmarks without rich environment feedback; (ii) mathematical reasoning on challenging problems with access to ground-truth solutions; and (iii) scaling behavior on coding tasks with rich execution-based feedback; and (iv) ablation study on the effectiveness of each component of the proposed method.

\subsection{Performance with Bootstrapping on Generated Responses as Feedback} 
\paragraph{Setting} We compare our method against SDPO \citep{hubotter2026reinforcement}, on-policy GRPO, which samples trajectories from the current policy, and GRPO, which reweights trajectories from previous policies using importance sampling. For all baselines, we follow the settings in \citep{hubotter2026reinforcement}. We evaluate on two sets of tasks: (1) scientific reasoning on undergraduate-level problems in chemistry, physics, biology, and materials science, using the L3 reasoning subset of SciKnowEval \citep{feng2024sciknoweval}; and (2) tool use, i.e., mapping a tool-API specification and user request to the correct tool call, using ToolAlpaca \citep{tang2023toolalpaca}. We split the data into training and test sets to evaluate in-domain generalization. We initialize the policies from Qwen3-8B \citep{yang2025qwen3} and OLMo-3-7B-Instruct \citep{olmo2025olmo}, and report Avg@16 in \Cref{tab:performance on science} in terms of wall-clock training time. We also report Maj@16 accuracy and Best@16 accuracy in \Cref{fig: major acc on science} and \Cref{fig: best acc on science}. All experiments are conducted using the \texttt{verl} framework on NVIDIA H200 GPUs to ensure fair comparisons with SDPO \citep{hubotter2026reinforcement}. See \Cref{tab:sciknoweval_experiment_config} for detailed experimental configurations.

\begin{table}[htp]
  \centering
  \small
  \sisetup{detect-weight=true}
  \setlength{\tabcolsep}{5pt}
  \renewcommand{\arraystretch}{1.15}
  \newcommand{\secondbest}[1]{\multicolumn{1}{c}{\underline{#1}}}
  \caption{Performance on SciKnowEval and Tool Use benchmarks. We report Avg@16 accuracy. The best and second-best results are shown in bold and underlined, respectively.}
  \label{tab:performance on science}
  \resizebox{\textwidth}{!}{%
    \begin{tabular}{@{}l *{15}{S}@{}}
      \toprule
      & \multicolumn{3}{c}{Chemistry}
      & \multicolumn{3}{c}{Physics}
      & \multicolumn{3}{c}{Biology}
      & \multicolumn{3}{c}{Materials}
      & \multicolumn{3}{c}{Tool use} \\
      \cmidrule(lr){2-4}
      \cmidrule(lr){5-7}
      \cmidrule(lr){8-10}
      \cmidrule(lr){11-13}
      \cmidrule(lr){14-16}
      & {5h} & {10h} & {15h}
      & {5h} & {10h} & {15h}
      & {5h} & {10h} & {15h}
      & {5h} & {10h} & {15h}
      & {5h} & {10h} & {15h} \\
      \midrule

      \textbf{Qwen3-8B}
      & \multicolumn{3}{c}{41.2}
      & \multicolumn{3}{c}{59.2}
      & \multicolumn{3}{c}{30.8}
      & \multicolumn{3}{c}{58.9}
      & \multicolumn{3}{c}{57.5} \\

      \ + GRPO
      & 63.1 & 68.3 & 72.6
      & 62.7 & 72.0 & \secondbest{74.1}
      & 35.2 & \bfseries 52.5 & \secondbest{53.5}
      & 74.3 & 75.9 & \secondbest{79.3}
      & 62.6 & \bfseries 66.3 & \secondbest{66.3} \\

      \ + GRPO {\small (on-policy)}
      & 51.8 & 62.9 & 68.1
      & 60.8 & 62.2 & 62.6
      & 33.8 & 35.5 & 36.4
      & 70.4 & 73.1 & 74.8
      & 60.1 & 62.5 & 65.5 \\

      \ + SDPO  
      & \bfseries 80.5 & \secondbest{80.5} & \secondbest{80.5}
      & \secondbest{71.1} & \secondbest{73.4} & 73.4
      & \bfseries 50.1 & \secondbest{50.1} & 50.1
      & \secondbest{76.7} & \secondbest{78.3} & 78.3
      & \bfseries 63.8 & 63.8 & 63.8 \\

      \ + SR-OPSD (ours)
      & \secondbest{80.2} & \bfseries 82.8 & \bfseries 82.8
      & \bfseries 82.7 & \bfseries 82.7 & \bfseries 82.7
      & \secondbest{50.0} & 50.0 & \bfseries 54.0
      & \bfseries 77.7 & \bfseries 79.6 & \bfseries 80.9
      & \secondbest{63.6} & \secondbest{64.4} & \bfseries 67.5 \\

      \addlinespace[2pt]
      \midrule

      \textbf{Olmo3-7B-Instruct}
      & \multicolumn{3}{c}{22.8}
      & \multicolumn{3}{c}{37.7}
      & \multicolumn{3}{c}{16.2}
      & \multicolumn{3}{c}{36.7}
      & \multicolumn{3}{c}{39.3} \\

      \ + GRPO
      & 39.1 & 45.4 & 53.4
      & 58.9 & 61.4 & 62.8
      & 33.1 & 38.5 & 42.6
      & 71.4 & 73.3 & 74.2
      & 56.9 & \secondbest{61.0} & \secondbest{61.0} \\

      \ + GRPO {\small (on-policy)}
      & 37.5 & 40.3 & 44.4
      & 54.3 & 59.1 & 60.9
      & 29.8 & 31.6 & 34.9
      & 66.4 & 70.7 & 73.3
      & 55.3 & 57.0 & 59.6 \\

      \ + SDPO
      & \bfseries 79.6 & \secondbest{79.6} & \secondbest{79.6}
      & \secondbest{67.6} & \secondbest{70.8} & \secondbest{70.8}
      & \secondbest{46.7} &  \secondbest{46.7} &  \secondbest{48.4}
      & \bfseries 77.4 & \bfseries 77.4 & \bfseries 77.4
      & \bfseries 58.8 & 60.4 & 60.4 \\

      \ + SR-OPSD (ours)
      & \secondbest{78.7} & \bfseries 79.9 & \bfseries 79.9
      & \bfseries 68.1 & \bfseries 72.3 & \bfseries 74.8
      & \bfseries 50.1 & \bfseries 51.1 & \bfseries 51.1
      & \secondbest{76.2} & \secondbest{76.2} & \secondbest{76.2}
      & \secondbest{58.1} & \bfseries 61.4 & \bfseries 62.7 \\

      \bottomrule
    \end{tabular}%
  }
\end{table}

\begin{wrapfigure}{r}{0.62\textwidth}
    \centering
    \vspace{-0.8em}
    \includegraphics[width=\linewidth]
    {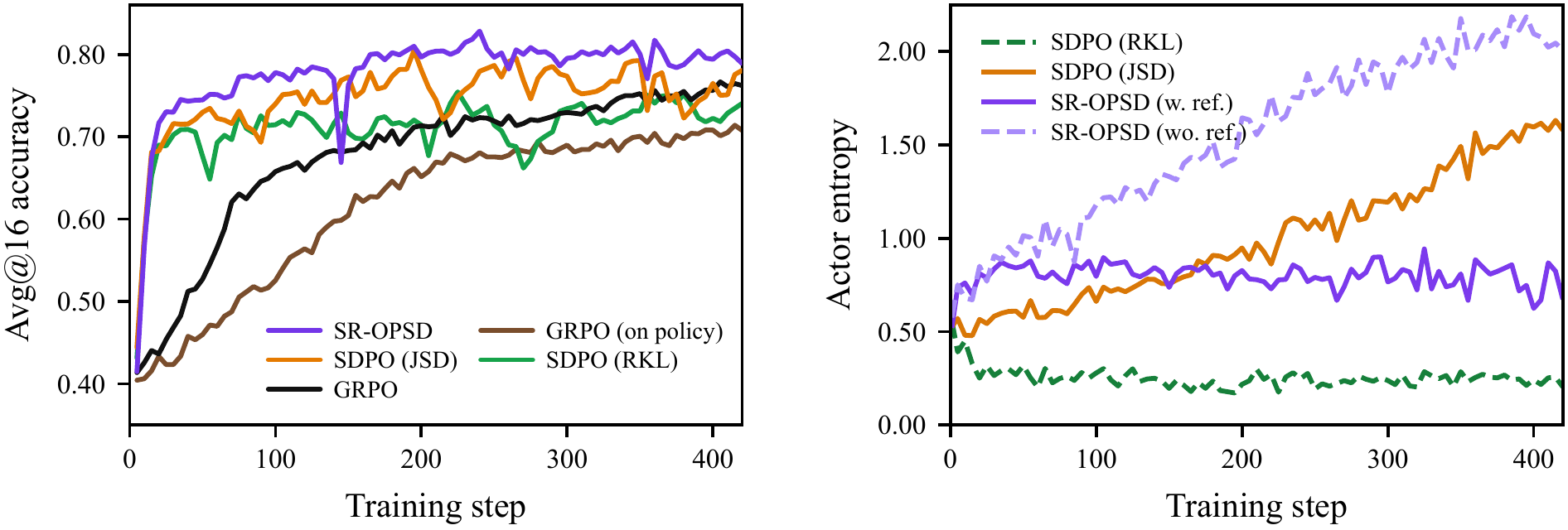}
    \caption{Investigation of Entropy of the Policy on
    Science Q\&A (Chemistry).
    Left: Avg@16 accuracy.
    Right: Training entropy of the policy.}
    \label{fig: illustration_SR-OPSD_results}
    \vspace{-0.8em}
\end{wrapfigure}

\paragraph{Results Analysis}
As shown in \Cref{tab:performance on science}, SR-OPSD achieves the highest $15$-hour Avg@16 accuracy in nine of the ten model--task settings and outperforms both GRPO variants in all ten. For Qwen3-8B, it leads across all five domains, achieving an average of $73.58$, compared with $69.22$ for SDPO and $69.16$ for GRPO. For OLMo3-7B-Instruct, it leads in four domains, with an average of $68.94$, exceeding SDPO ($67.32$) and GRPO ($58.80$). The only exception is OLMo3-7B-Instruct on Materials, where SR-OPSD trails SDPO by $1.2$ percentage points. Across the reported training budgets, the number of settings in which SR-OPSD ranks first increases from four at $5$ hours to seven at $10$ hours and nine at $15$ hours.

\begin{figure}[htb]
    \centering
    \includegraphics[width=1.\linewidth]{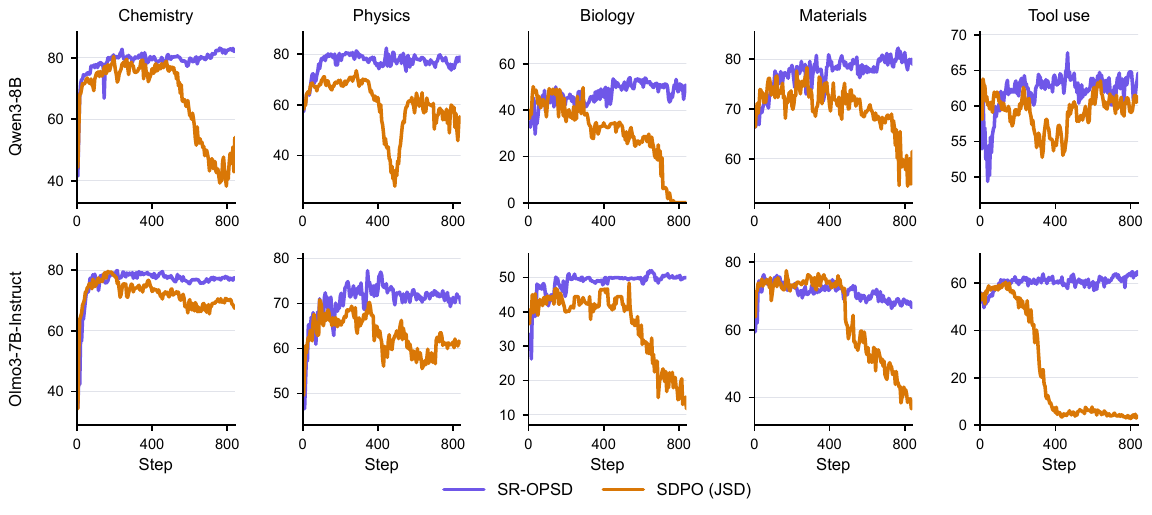}
\caption{Training dynamics of SDPO and SR-OPSD: the x-axis shows the training step, and the y-axis shows the test Avg@16 accuracy.}
\label{fig: mean acc eval performance of table 1}
\vspace{-10pt}
\end{figure}

In \Cref{fig: mean acc eval performance of table 1}, we report Avg@16 test accuracy over training steps on five reasoning-related benchmarks for Qwen3-8B and Olmo3-7B-Instruct.
Compared with the SDPO (with JSD) baseline, SR-OPSD generally maintains more stable long-budget performance, whereas SDPO often reaches an early peak and then degrades. These results suggest that the proposed method provides a more favorable performance and training stability.

\subsection{Performance on Mathematical Reasoning Benchmarks} 
\paragraph{Setting} We then evaluate the performance in a setting where we have access to ground-truth solutions as feedback on very hard mathematical reasoning problems. Following the setting of \citet{agrawal2026reinforcementlearningrichfeedback}, we use a training collection of 758 particularly challenging mathematics problems. We benchmark our approach against OPSD (forward KL) \citep{zhao2026selfdistilledreasoner}, SDPO  \citep{hubotter2026reinforcement} (reverse KL), and GRPO on AIME 2024, AIME 2025, HMMT 2025, AMC 2023, and Minerva, with all methods initialized from Qwen3-4B-Instruct-2507 or Qwen3-8B. See \Cref{tab:math_experiment_config} for detailed experimental configurations.

\begin{table}[htp]
\centering
\caption{Performance on mathematical reasoning benchmarks.
The best and second-best results within each model size
are shown in bold and underlined.}
\label{tab:performance_on_math}
\resizebox{\textwidth}{!}{%
\begin{tabular}{@{}ll*{10}{c}@{}}
\toprule
\multirow{2}{*}{Model}
& \multirow{2}{*}{Method}
& \multicolumn{2}{c}{AIME24}
& \multicolumn{2}{c}{AIME25}
& \multicolumn{2}{c}{HMMT25}
& \multicolumn{2}{c}{AMC23}
& \multicolumn{2}{c}{Minerva} \\
\cmidrule(lr){3-4}
\cmidrule(lr){5-6}
\cmidrule(lr){7-8}
\cmidrule(lr){9-10}
\cmidrule(lr){11-12}
& & Avg@16 & Pass@16
  & Avg@16 & Pass@16
  & Avg@16 & Pass@16
  & Avg@16 & Pass@16
  & Avg@16 & Pass@16 \\
\midrule
\multirow{5}{*}{Qwen3-4B}
& Base & 61.4 & 82.0 & 50.3 & 69.5 & 30.3 & 48.8 & \textbf{93.8} & \underline{98.8} & 43.2 & 48.6\\
& $+$ GRPO
& 62.5 & \textbf{86.7}
& 51.1 & \textbf{80.0}
& 34.3 & 60.0
& 89.7 & \textbf{100.0}
& 46.7 & \textbf{59.2} \\
& $+$ SDPO 
& \underline{65.0} & \textbf{86.7}
& \underline{54.0} & \textbf{80.0}
& \underline{36.3} & \underline{63.3}
& 92.2 & \textbf{100.0}
& \underline{47.0} & 58.5 \\
& $+$ OPSD 
& 61.9 & \underline{83.3}
& 52.1 & \underline{73.3}
& 33.5 & 56.7
& 89.8 & \textbf{100.0}
& 46.9 & \underline{58.8} \\
& $+$ SR-OPSD (Ours)
& \textbf{66.3} & \textbf{86.7}
& \textbf{54.6} & \textbf{80.0}
& \textbf{36.5} & \textbf{66.7}
& \underline{93.1} & \textbf{100.0}
& \textbf{47.4} & \textbf{59.2}  \\
\midrule
\multirow{5}{*}{Qwen3-8B}
& Base & 75.9 & \underline{86.8} & 66.9 & \underline{79.6} & 44.7 & 65.6 & 95.6 & \textbf{100.0} & 48.9 & 56.3 \\
& $+$  GRPO
& 76.6 & \textbf{93.3}
& 70.1 & \textbf{93.3}
& 46.0 & \underline{80.0}
& 96.2 & \textbf{100.0}
& 49.2 & 61.0 \\
& $+$  SDPO 
& 77.1 & \textbf{93.3}
& 70.0 & \textbf{93.3}
& 45.8 & 76.7
& \underline{96.7} & \textbf{100.0}
& \underline{49.7} & 60.7 \\
& $+$   OPSD
& \underline{77.2} & \textbf{93.3}
& \underline{70.3} & \textbf{93.3}
& \underline{46.2} & \textbf{83.3}
& \textbf{97.3} & \textbf{100.0}
& 49.4 & \underline{61.4} \\
& $+$  SR-OPSD (Ours)
& \textbf{77.7} & \textbf{93.3}
& \textbf{71.9} & \textbf{93.3}
& \textbf{46.5} & \underline{80.0}
& \textbf{97.3} & \textbf{100.0}
& \textbf{50.0} & \textbf{63.2} \\
\bottomrule
\end{tabular}%
}
\end{table}

\paragraph{Results analysis}
Table~\ref{tab:performance_on_math} shows that SR-OPSD achieves the highest mean Avg@16 across the five benchmarks at both model sizes: 59.58 for Qwen3-4B and 68.68 for Qwen3-8B. These scores exceed SDPO by 0.68 and 0.82 percentage points
and OPSD by 2.74 and 0.60 percentage points, respectively. For Qwen3-4B, SR-OPSD outperforms all three trained baselines in Avg@16 on every benchmark and achieves the best or joint-best Pass@16 on all five. Its mean Pass@16 reaches 78.52, compared with 77.70 for SDPO and 74.42 for OPSD, showing gains in both
average accuracy and multi-sample problem coverage. For Qwen3-8B, SR-OPSD achieves the highest Avg@16 on four benchmarks and ties with OPSD on AMC23. On AIME25, it reaches 71.9, exceeding the strongest
baseline, OPSD, by 1.6 percentage points. It also achieves the best or joint-best Pass@16 on four benchmarks, including 63.2 on Minerva, 1.8 percentage
points above OPSD. On HMMT25, SR-OPSD and OPSD obtain Pass@16 scores of
80.0 and 83.3, respectively. Overall, SR-OPSD improves Avg@16 over GRPO and SDPO across all five tasks at both model sizes, while matching or exceeding OPSD. Its best or joint-best Pass@16 in nine of the ten model--benchmark combinations further shows that these
broad gains in average accuracy are accompanied by strong multi-sample coverage.

\subsection{Performance on Code generation under Model Scaling} 

\paragraph{Setting} We next assess our method on code-generation tasks. Programming provides a representative setting with informative feedback signals, including execution errors and unsuccessful unit tests. Success in this domain depends on effective credit assignment, since the model must locate the precise source of an error and use that information to avoid making the same mistake in subsequent attempts. Following the protocol of \citep{hubotter2026reinforcement}, we focus our evaluation on the LCBv6 subset. We use the Qwen3 model family for this investigation and train all methods for 100 iterations. Results are averaged over two random seeds, and are reported with standard deviations. See \Cref{tab: experiments config for LCB task} for detailed experimental configurations.

\begin{wraptable}{r}{0.6\textwidth}
\vspace{-5pt}
\centering
\caption{Performance on LiveCodeBench v6 across Qwen3 model scales. The best result in each column is shown in bold.}
\footnotesize
\setlength{\tabcolsep}{3.2pt}
\renewcommand{\arraystretch}{1.05}
\begin{tabular*}{\linewidth}{
@{\extracolsep{\fill}}lcccc@{}
}
\toprule
Method & Qwen3-0.6B & Qwen3-1.7B & Qwen3-4B & Qwen3-8B \\
\midrule
Base & 9.2 & 14.7 & 23.1 & 27.9 \\
GRPO & $15.1 \pm 0.6$ & $23.7\pm0.9$ & $34.8 \pm 2.2$ & $37.1 \pm 2.2$ \\
SDPO & $\mathbf{18.4} \pm 1.1$ & $27.8 \pm 1.4$ & $33.5 \pm 0.1$ & $44.2 \pm 0.2$ \\
SR-OPSD & $16.9 \pm 1.0$ & $\mathbf{31.8} \pm 0.3$ & $\mathbf{44.9} \pm 0.7$ & $\mathbf{48.0} \pm 2.9$ \\
\bottomrule
\end{tabular*}
\label{tab:qwen3_scaling}
\end{wraptable}

\paragraph{Results Analysis} 
As shown in Table~\ref{tab:qwen3_scaling}, SR-OPSD improves over the
base models at all four scales, with gains ranging from 7.7 to 21.8
percentage points. Its mean score increases monotonically from 16.9
at 0.6B to 48.0 at 8B. At 1.7B, 4B, and 8B, SR-OPSD achieves the
highest mean scores, exceeding the strongest baseline at each scale
by 4.0, 10.1, and 3.8 percentage points, respectively.

\subsection{Ablation Studies} 
\label{sec:ablations}
\begin{wraptable}{r}{0.5\textwidth}
\vspace{-15pt}
\centering
\caption{\small
Ablation on SciKnowEval Physics (Avg@16).
Top: effect of self-reference, with $\alpha=0.90$ for self-referenced variants;
$\Delta$ denotes the change in percentage points.
Bottom: joint effect of $\alpha$ and $\rho$.
The best result in each block is bold.
}
\label{tab:ablation_components}
\scriptsize
\setlength{\tabcolsep}{3pt}
\renewcommand{\arraystretch}{1.08}

\begin{tabular*}{\linewidth}{
@{\extracolsep{\fill}}lrrr@{}
}
\toprule
\multicolumn{4}{@{}l}{
\textit{Effect of reference anchoring}
}\\[2pt]
Objective
& \multicolumn{1}{c}{No reference}
& \multicolumn{1}{c}{With reference}
& \multicolumn{1}{c}{$\Delta$} \\
\midrule
JSD
& 73.4
& 77.2
& $+3.8$ \\
Forward KL
& 79.1
& 76.3
& $-2.8$ \\
Forward R\'enyi ($\rho=0.50$)
& 76.5
& 79.9
& $+3.4$ \\
Forward R\'enyi ($\rho=0.70$)
& 75.9
& \textbf{82.7}
& $+6.8$ \\
Forward R\'enyi ($\rho=0.90$)
& 79.7
& 77.2
& $-2.5$ \\
Forward R\'enyi ($\rho=0.95$)
& 77.7
& \underline{81.2}
& $+3.5$ \\
\end{tabular*}

\par\vspace{3pt}

\begin{tabular*}{\linewidth}{
@{\extracolsep{\fill}}lcccc@{}
}
\midrule
\multicolumn{5}{@{}l}{
\textit{Joint effect of $\alpha$ and $\rho$}
}\\[2pt]
& \multicolumn{4}{c}{R\'enyi order $\rho$} \\
\cmidrule(lr){2-5}
$\alpha$
& 0.50
& 0.70
& 0.90
& 0.95 \\
\midrule
0.50
& 78.8
& 78.8
& 78.2
& 78.4 \\
0.70
& 76.3
& 82.0
& 77.0
& 80.7 \\
0.90
& 79.9
& \textbf{82.7}
& 77.2
& 81.2 \\
\bottomrule
\end{tabular*}
\vspace{-13pt}
\end{wraptable}
We examine how the effect of reference anchoring depends on the projection objective, and how target interpolation interacts with the R\'enyi order. Using Qwen3-8B on SciKnowEval Physics, we first compare JSD \citep{hubotter2026reinforcement}, forward KL \citep{zhao2026selfdistilledreasoner}, and forward R\'enyi with and without reference anchoring. We then jointly vary $\alpha$ and $\rho$ within SR-OPSD to examine their effects on performance, as shown in \Cref{tab:ablation_components}. All variants share the same initialization, training data, and evaluation protocol. For each configuration, we report the Avg@16 accuracy achieved within a 10-hour training budget.

\vspace{-5pt}
\paragraph{Does self-reference consistently improve performance?} The effect of reference anchoring depends on the projection objective. Self-reference improves JSD from 73.4 to 77.2, a gain of 3.8 percentage points, but reduces forward KL from 79.1 to 76.3, a drop of 2.8 points. These opposing effects show that adding a reference target does not consistently improve performance, motivating consideration of how the student is projected toward that target.

\vspace{-5pt}
\paragraph{Does R\'enyi projection alone improve over forward KL?} Without self-reference, forward R\'enyi achieves 75.9 at $\rho=0.70$ and 77.7 at $\rho=0.95$, both below forward KL at 79.1. Adding self-reference raises these scores to 82.7 and 81.2, respectively. In particular, at $\rho=0.70$, self-reference alone and R\'enyi projection alone each underperform the unanchored forward-KL baseline, whereas their combination exceeds it by 3.6 percentage points. This reversal supports jointly designing the supervisory target and projection geometry.

\vspace{-5pt}
\paragraph{How do $\alpha$ and $\rho$ interact?}
The effect of changing $\alpha$ depends on $\rho$. Increasing $\alpha$ from 0.50 to 0.90 improves Avg@16 from 78.8 to 82.7 at $\rho=0.70$, but reduces it from 78.2 to 77.2 at $\rho=0.90$. Sensitivity to $\rho$ also varies with $\alpha$: scores span only 78.2--78.8 at $\alpha=0.50$, compared with 77.2--82.7 at $\alpha=0.90$. The best tested configuration is $(\alpha,\rho)=(0.90,0.70)$, achieving 82.7, followed by $(0.70,0.70)$ at 82.0. In summary, neither self-reference nor R\'enyi projection alone consistently improves performance. See \Cref{fig:physics_response_lengths} in Appendix~\ref{appendix:Extra Details of Ablation Study} for discussions on responses lengths.

\vspace{-5pt}
\paragraph{Memory and Computation Cost} We separately benchmark resource costs on SciKnowEval Physics with Qwen3-8B, giving each run exclusive access to a single node with eight NVIDIA H200 GPUs. For each run, we report the maximum logged single-device PyTorch-allocated memory across all GPU workers and steps in decimal GB, and compute the average step time. 
\begin{wrapfigure}{r}{0.6\textwidth}
    \vspace{-5pt}
    \centering
    \includegraphics[width=0.93\linewidth]{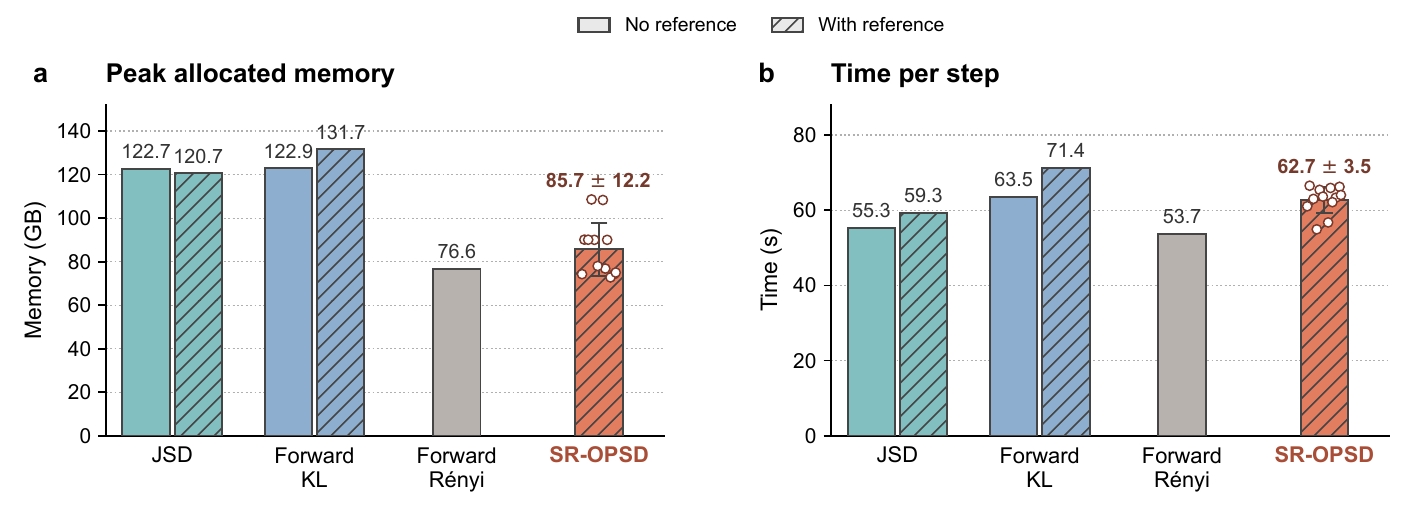}
    \caption{Memory Consumption and Empirical Computational Cost.}
    \label{fig:mem_cost}
\vspace{-10pt}
\end{wrapfigure}
Across the 12 evaluated $(\alpha,\rho)$ configurations, SR-OPSD uses $85.7 \pm 12.2$ GB and $62.7 \pm 3.5$ seconds per step, where $\pm$ denotes the population standard deviation across configurations. Compared with self-reference-anchored JSD and forward KL, its mean allocated memory is lower by $29.0\%$ and $34.9\%$, respectively, while its mean step time is $5.7\%$ higher and $12.2\%$ lower. Relative to unanchored forward R\'enyi, SR-OPSD uses $11.9\%$ more memory and takes $16.8\%$ longer per step.

\section{Related Work}

\paragraph{On-Policy Distillation}
On-policy distillation uses student-generated trajectories for training, allowing the teacher policy to provide dense supervision on prefixes visited by the student \citep{agarwal2023onpolicy, wen-etal-2023-f, gu2023minillm,  lu2025onpolicydistillation, jin2026entropyaware, luo2026demystifying, yang2026learningbeyondteacher, stein2026gates, oh2026kl, kim2026distillation}. \citet{gu2023minillm} optimizes sequence-level reverse KL through policy-gradient estimators, while \citet{agarwal2023onpolicy} evaluates forward KL, reverse KL, and JSD on student rollouts. Later work studies entropy-dependent divergence selection, reference-policy constraints, and rollout mixing to improve stability \citep{jin2026entropyaware,luo2026demystifying}. \citet{yang2026learningbeyondteacher} further formulates external-teacher on-policy distillation as KL-regularized policy optimization, where a teacher-to-reference log-probability ratio defines the dense reward and controls interpolation or extrapolation between the two policies. See \citep{song2026survey} for a detailed review.

\vspace{-5pt}
\paragraph{On-Policy Self-Distillation}
These methods replace the external teacher with a self-teacher policy derived from the policy being optimized and conditioned on privileged or auxiliary context available only during training. Existing work uses ground-truth solutions, environment feedback, successful sibling responses, or document context to construct the self-teacher \citep{zhao2026selfdistilledreasoner,hubotter2026reinforcement,penaloza2026pidistill, ye2026policy, sang2026crisp, zhao2026training, wang2026trace}. Subsequent work controls the amount of additional information exposed to the teacher, restricts distillation to selected spans, or gates token-level supervision according to teacher reliability and uncertainty \citep{han2026atesd,wang2026trace,liu2026whenteachertokensreliable,ke2026egrsd}. Other studies analyze the effects of additional context information, Top-$K$ approximation, and update schemes of self-teacher policies \citep{zhu2026manyfaces,guo2026whenshouldteacher}. Existing objectives mainly use forward KL, reverse KL, or JSD for direct distribution matching \citep{zhao2026selfdistilledreasoner,hubotter2026reinforcement}, while preference-based self-distillation replaces direct matching with reward-regularized preference optimization \citep{yu2026preferencebased}.

\section{Conclusion}

Our study shows that reference anchoring can improve or degrade self-distillation accuracy depending on the projection objective. This dependence highlights the need to consider target construction and the student's response to the target jointly. SR-OPSD brings these choices into a single objective by combining a geometric target anchored to the frozen initial policy with forward R\'enyi projection. At fixed student-generated contexts, our variational characterization and exact token-logit gradient reveal an effective update target that geometrically combines the self-teacher, reference, and current
student policies. Experiments across scientific reasoning, tool use, mathematical reasoning, and code generation demonstrate strong performance.

Understanding how this interaction evolves throughout training remains an open question. Our analysis conditions on fixed rollout contexts and frozen target components, leaving the joint evolution of the student, self-teacher, and rollout distribution beyond its scope.
Future work could extend the analysis to these coupled dynamics and adapt $\alpha$ and $\rho$ to teacher reliability or policy drift.

\bibliography{iclr2027_conference}
\bibliographystyle{iclr2027_conference}

\newpage
\appendix
\section*{Appendix}

\section{Notation Table}
\label{appendix:notation}

\begin{table}[h]
\centering
\scriptsize
\renewcommand{\arraystretch}{1.08}
\caption{Summary of the notation.}
\begin{tabular}{p{0.20\textwidth}p{0.72\textwidth}}
\toprule
\textbf{Symbol} & \textbf{Meaning} \\
\midrule
$x,\mathcal D$ & Prompt and prompt distribution, with $x\sim\mathcal D$. \\
$y,y_t,y_{<t}$ & On-policy response, token at position $t$, and prefix before $t$. \\
$a,\mathcal V,\mathcal C_t$ & Generic next-token action, full vocabulary, and candidate set at position $t$. \\
$h_t$ & Student prefix context, $h_t=(x,y_{<t})$. \\
$h_t^{\mathrm{sf}},h_t^{\mathrm{ref}}$ & Context information for the self-teacher policy and reference policy at position $t$, respectively. \\
$h,h_{\mathrm{sf}},h_{\mathrm{ref}}$ & Fixed-context shorthand obtained by suppressing the position index $t$. \\
$\pi_\theta,\theta$ &  Policy to be optimized (student policy) and its parameters. \\
$\pi_{\bar\theta},\bar\theta$ & Frozen or EMA self-teacher policy and its parameters. \\
$\pi_{\theta_{\mathrm{old}}}$ & Rollout policy used to sample the current on-policy batch. \\
$\pi_{\mathrm{ref}}$ & Reference policy.\\
$\theta_0$ &  Parameters of the student policy before optimization.\\
$\pi_T$ & External teacher policy, when on-policy distillation is discussed. \\
$\operatorname{sg}[\cdot]$ & Stop-gradient operator. \\
$\pi_\alpha^\star,Z_\alpha$ & Reference-anchored self-teacher target and its normalizing constant. \\
$\alpha,\rho$ & Target-interpolation coefficient and R\'enyi order. \\
$D_{\mathrm{KL}},D_\rho$ & KL divergence and R\'enyi divergence. \\
$\beta$ & EMA coefficient. \\
\bottomrule
\end{tabular}
\label{tab:notation}
\end{table}

\clearpage
\section{Theoretical Results and Proofs}
\label{appendix:Theoretical Results and Proofs}

\subsection{Proof of Token-Level Variational Characterization of On-policy Self-distillation}
\label{appendix: proof of variational characterization of opsd}

\begin{proposition}[Token-Level Variational Characterization of OPSD]
\label{prop:opsd_kl_regularized_equivalence}
Fix $h$, $h_{\mathrm{sf}}$, and $h_{\mathrm{ref}}$, and let
$\alpha\in[0,1]$. Assume that, for every token $a\in\mathcal V$, $\pi_\theta(a\mid h)>0,
\qquad
\operatorname{sg}\!\left[
\pi_{\bar\theta}(a\mid h_{\mathrm{sf}})
\right]>0,
\qquad
\pi_{\mathrm{ref}}(a\mid h_{\mathrm{ref}})>0$.
Define $Z_\alpha(h_{\mathrm{sf}},h_{\mathrm{ref}})
:=
\sum_{a\in\mathcal V}
\operatorname{sg}\!\left[
\pi_{\bar\theta}(a\mid h_{\mathrm{sf}})
\right]^\alpha
\pi_{\mathrm{ref}}(a\mid h_{\mathrm{ref}})^{1-\alpha}$
and $\pi_\alpha^\star(a\mid h_{\mathrm{sf}},h_{\mathrm{ref}})
:=
\frac{
\operatorname{sg}\!\left[
\pi_{\bar\theta}(a\mid h_{\mathrm{sf}})
\right]^\alpha
\pi_{\mathrm{ref}}(a\mid h_{\mathrm{ref}})^{1-\alpha}
}{
Z_\alpha(h_{\mathrm{sf}},h_{\mathrm{ref}})
}$. Then the conditional token-level objective
\begin{align*}
\mathcal J_\alpha
(\theta;h,h_{\mathrm{sf}},h_{\mathrm{ref}})
:=
\mathbb E_{a\sim\pi_\theta(\cdot\mid h)}
\left[
\alpha\log
\frac{
\operatorname{sg}\!\left[
\pi_{\bar\theta}(a\mid h_{\mathrm{sf}})
\right]
}{
\pi_{\mathrm{ref}}(a\mid h_{\mathrm{ref}})
}
\right]
-
D_{\mathrm{KL}}\!\left(
\pi_\theta(\cdot\mid h)
\,\middle\|\,
\pi_{\mathrm{ref}}(\cdot\mid h_{\mathrm{ref}})
\right)
\end{align*}
satisfies $\mathcal J_\alpha
(\theta;h,h_{\mathrm{sf}},h_{\mathrm{ref}})
=
-
D_{\mathrm{KL}}\!\left(
\pi_\theta(\cdot\mid h)
\,\middle\|\,
\pi_\alpha^\star(\cdot\mid h_{\mathrm{sf}},h_{\mathrm{ref}})
\right)+
\log Z_\alpha(h_{\mathrm{sf}},h_{\mathrm{ref}})$.
Consequently, for fixed contexts and frozen target components, maximizing
$\mathcal J_\alpha$ with respect to $\theta$ is equivalent to minimizing the
reverse KL from the student to $\pi_\alpha^\star$. When $\alpha=1$,
$Z_1=1$ and
$\pi_1^\star(\cdot\mid h_{\mathrm{sf}},h_{\mathrm{ref}})
=
\operatorname{sg}[\pi_{\bar\theta}(\cdot\mid h_{\mathrm{sf}})]$, recovering
the token-level objective.
\end{proposition}


\begin{proof}
Fix $h$, $h_{\mathrm{sf}}$, and $h_{\mathrm{ref}}$, and abbreviate
\[
q(a):=\pi_\theta(a\mid h),
\qquad
p(a):=\operatorname{sg}\!\left[
\pi_{\bar\theta}(a\mid h_{\mathrm{sf}})
\right],
\qquad
r(a):=\pi_{\mathrm{ref}}(a\mid h_{\mathrm{ref}}).
\]
The conditional objective expands as
\begin{align*}
\mathcal J_\alpha
&=
\sum_a q(a)
\left[
\alpha\log\frac{p(a)}{r(a)}
-
\log\frac{q(a)}{r(a)}
\right]
\\
&=
\sum_a q(a)
\left[
-\log q(a)
+
\alpha\log p(a)
+
(1-\alpha)\log r(a)
\right].
\end{align*}
By definition,
\[
\log\pi_\alpha^\star(a\mid h_{\mathrm{sf}},h_{\mathrm{ref}})
=
\alpha\log p(a)
+
(1-\alpha)\log r(a)
-
\log Z_\alpha(h_{\mathrm{sf}},h_{\mathrm{ref}}).
\]
Substituting this identity yields
\begin{align*}
\mathcal J_\alpha
&=
\sum_a q(a)
\left[
-\log q(a)
+
\log\pi_\alpha^\star(a\mid h_{\mathrm{sf}},h_{\mathrm{ref}})
+
\log Z_\alpha(h_{\mathrm{sf}},h_{\mathrm{ref}})
\right]
\\
&=
-
D_{\mathrm{KL}}\!\left(
\pi_\theta(\cdot\mid h)
\,\middle\|\,
\pi_\alpha^\star(\cdot\mid h_{\mathrm{sf}},h_{\mathrm{ref}})
\right)
+
\log Z_\alpha(h_{\mathrm{sf}},h_{\mathrm{ref}}),
\end{align*}
where the last step uses $\sum_a q(a)=1$. Since the contexts and target
components are fixed, $Z_\alpha$ is constant with respect to the inner
optimization variable $\theta$.

When $\alpha=1$,
\[
Z_1(h_{\mathrm{sf}},h_{\mathrm{ref}})
=
\sum_a p(a)
=
1
\]
and
\[
\pi_1^\star(a\mid h_{\mathrm{sf}},h_{\mathrm{ref}})
=
p(a)
=
\operatorname{sg}\!\left[
\pi_{\bar\theta}(a\mid h_{\mathrm{sf}})
\right].
\]
The objective therefore reduces to the reverse-KL OPSD objective. The
constancy statement is conditional on the fixed contexts: after averaging over
a context distribution that itself changes with the policy, the expected
normalizer need not be constant.
\end{proof}

\subsection{Proof of Proposition
\ref{prop:opsd_forward_renyi_regularized_equivalence}}
\label{appendix:proof_of_opsd_forward_renyi_regularized_equivalence}

\begin{proof}
Fix $h$, $h_{\mathrm{sf}}$, and $h_{\mathrm{ref}}$. To simplify the
notation, define
\begin{align*}
q(a)
&:=
\pi_\theta(a\mid h),
\\
p(a)
&:=
\operatorname{sg}\!\left[
\pi_{\bar\theta}(a\mid h_{\mathrm{sf}})
\right],
\\
r(a)
&:=
\pi_{\mathrm{ref}}(a\mid h_{\mathrm{ref}}).
\end{align*}
Because all token probabilities are strictly positive, all logarithms,
probability ratios, and real-valued powers appearing below are well-defined.

Under these abbreviations, the normalizing constant is
\begin{align*}
Z_\alpha:=
Z_\alpha(h_{\mathrm{sf}},h_{\mathrm{ref}})=\sum_{a\in\mathcal V}
p(a)^\alpha r(a)^{1-\alpha}.
\end{align*}
The normalized self-referenced target satisfies
\begin{align*}
\pi_\alpha^\star
(a\mid h_{\mathrm{sf}},h_{\mathrm{ref}}) =
\frac{
p(a)^\alpha r(a)^{1-\alpha}
}{
Z_\alpha
}.
\end{align*}
Multiplying both sides of above equation by $Z_\alpha$ gives
\begin{align}
p(a)^\alpha r(a)^{1-\alpha}
=
Z_\alpha
\pi_\alpha^\star
(a\mid h_{\mathrm{sf}},h_{\mathrm{ref}}).
\label{eq:proof_unnormalized_target_identity}
\end{align}

We first simplify the numerator of $\mathcal{R}_{\alpha, \rho}$.
By the definition of the pseudo-reward $r_\alpha$,
\begin{align*}
r_\alpha(a;h_{\mathrm{sf}},h_{\mathrm{ref}})
=
\alpha\log\frac{p(a)}{r(a)}.
\end{align*}
Multiplying both sides by $\rho$ yields
\begin{align}
\rho
r_\alpha(a;h_{\mathrm{sf}},h_{\mathrm{ref}})
=
\rho\alpha
\log\frac{p(a)}{r(a)}.
\label{eq:proof_scaled_pseudo_reward}
\end{align}
Exponentiating both sides of
\Cref{eq:proof_scaled_pseudo_reward} gives
\begin{align}
\exp\!\left(
\rho
r_\alpha(a;h_{\mathrm{sf}},h_{\mathrm{ref}})
\right)
&=
\exp\!\left(
\rho\alpha
\log\frac{p(a)}{r(a)}
\right)
\nonumber\\
&=
\left(
\frac{p(a)}{r(a)}
\right)^{\rho\alpha}
\nonumber\\
&=
p(a)^{\rho\alpha}
r(a)^{-\rho\alpha}.
\label{eq:proof_exponentiated_pseudo_reward}
\end{align}

Consider one summand in the numerator of $\mathcal R_{\alpha,\rho}$.
Using \Cref{eq:proof_exponentiated_pseudo_reward}, we obtain
\begin{align}
&r(a)^\rho q(a)^{1-\rho}
\exp\!\left(
\rho
r_\alpha(a;h_{\mathrm{sf}},h_{\mathrm{ref}})
\right)
\nonumber\\
&=
r(a)^\rho q(a)^{1-\rho}
p(a)^{\rho\alpha}
r(a)^{-\rho\alpha}.
\label{eq:proof_numerator_step_one}
\end{align}
Combining the two powers of $r(a)$ gives
\begin{align}
r(a)^\rho r(a)^{-\rho\alpha}
&=
r(a)^{\rho-\rho\alpha}
\nonumber\\
&=
r(a)^{\rho(1-\alpha)}.
\label{eq:proof_reference_power}
\end{align}
Substituting \Cref{eq:proof_reference_power} into
\Cref{eq:proof_numerator_step_one} yields
\begin{align}
&r(a)^\rho q(a)^{1-\rho}
\exp\!\left(
\rho
r_\alpha(a;h_{\mathrm{sf}},h_{\mathrm{ref}})
\right)
\nonumber\\
&=
q(a)^{1-\rho}
p(a)^{\rho\alpha}
r(a)^{\rho(1-\alpha)}.
\label{eq:proof_numerator_step_two}
\end{align}
Therefore, \Cref{eq:proof_numerator_step_two} becomes
\begin{align}
&r(a)^\rho q(a)^{1-\rho}
\exp\!\left(
\rho
r_\alpha(a;h_{\mathrm{sf}},h_{\mathrm{ref}})
\right)
\nonumber\\
&=
q(a)^{1-\rho}
\left[
p(a)^\alpha r(a)^{1-\alpha}
\right]^\rho.
\label{eq:proof_numerator_step_three}
\end{align}
Applying \Cref{eq:proof_unnormalized_target_identity} to
\Cref{eq:proof_numerator_step_three} gives
\begin{align}
&r(a)^\rho q(a)^{1-\rho}
\exp\!\left(
\rho
r_\alpha(a;h_{\mathrm{sf}},h_{\mathrm{ref}})
\right)
\nonumber\\
&=
q(a)^{1-\rho}
\left[
Z_\alpha
\pi_\alpha^\star
(a\mid h_{\mathrm{sf}},h_{\mathrm{ref}})
\right]^\rho
\nonumber\\
&=
Z_\alpha^\rho
\pi_\alpha^\star
(a\mid h_{\mathrm{sf}},h_{\mathrm{ref}})^\rho
q(a)^{1-\rho}.
\label{eq:proof_numerator_step_four}
\end{align}

Summing both sides of \Cref{eq:proof_numerator_step_four} over
$a\in\mathcal V$ gives
\begin{align}
&\sum_{a\in\mathcal V}
r(a)^\rho q(a)^{1-\rho}
\exp\!\left(
\rho
r_\alpha(a;h_{\mathrm{sf}},h_{\mathrm{ref}})
\right)
\nonumber\\
&=
\sum_{a\in\mathcal V}
Z_\alpha^\rho
\pi_\alpha^\star
(a\mid h_{\mathrm{sf}},h_{\mathrm{ref}})^\rho
q(a)^{1-\rho}.
\label{eq:proof_summed_numerator_step_one}
\end{align}
Because $Z_\alpha$ does not depend on the summation variable $a$, it can
be taken outside the sum:
\begin{align}
&\sum_{a\in\mathcal V}
r(a)^\rho q(a)^{1-\rho}
\exp\!\left(
\rho
r_\alpha(a;h_{\mathrm{sf}},h_{\mathrm{ref}})
\right)
\nonumber\\
&=
Z_\alpha^\rho
\sum_{a\in\mathcal V}
\pi_\alpha^\star
(a\mid h_{\mathrm{sf}},h_{\mathrm{ref}})^\rho
q(a)^{1-\rho}.
\label{eq:proof_summed_numerator_step_two}
\end{align}

By the definition of R\'enyi divergence,
\begin{align}
&D_\rho\!\left(
\pi_\alpha^\star(\cdot\mid h_{\mathrm{sf}},h_{\mathrm{ref}})
\,\middle\|\,
q
\right)
\nonumber\\
&=
\frac{1}{\rho-1}
\log
\sum_{a\in\mathcal V}
\pi_\alpha^\star
(a\mid h_{\mathrm{sf}},h_{\mathrm{ref}})^\rho
q(a)^{1-\rho}.
\label{eq:proof_target_renyi_definition}
\end{align}
Multiplying both sides of
\Cref{eq:proof_target_renyi_definition} by $\rho-1$ gives
\begin{align}
&(\rho-1)
D_\rho\!\left(
\pi_\alpha^\star(\cdot\mid h_{\mathrm{sf}},h_{\mathrm{ref}})
\,\middle\|\,
q
\right)
\nonumber\\
&=
\log
\sum_{a\in\mathcal V}
\pi_\alpha^\star
(a\mid h_{\mathrm{sf}},h_{\mathrm{ref}})^\rho
q(a)^{1-\rho}.
\label{eq:proof_target_renyi_log}
\end{align}
Exponentiating both sides of
\Cref{eq:proof_target_renyi_log} gives
\begin{align}
&\exp\!\left(
(\rho-1)
D_\rho\!\left(
\pi_\alpha^\star(\cdot\mid h_{\mathrm{sf}},h_{\mathrm{ref}})
\,\middle\|\,
q
\right)
\right)
\nonumber\\
&=
\sum_{a\in\mathcal V}
\pi_\alpha^\star
(a\mid h_{\mathrm{sf}},h_{\mathrm{ref}})^\rho
q(a)^{1-\rho}.
\label{eq:proof_target_renyi_exponential}
\end{align}
Substituting \Cref{eq:proof_target_renyi_exponential} into
\Cref{eq:proof_summed_numerator_step_two} yields
\begin{align}
&\sum_{a\in\mathcal V}
r(a)^\rho q(a)^{1-\rho}
\exp\!\left(
\rho
r_\alpha(a;h_{\mathrm{sf}},h_{\mathrm{ref}})
\right)
\nonumber\\
&=
Z_\alpha^\rho
\exp\!\left(
(\rho-1)
D_\rho\!\left(
\pi_\alpha^\star(\cdot\mid h_{\mathrm{sf}},h_{\mathrm{ref}})
\,\middle\|\,
q
\right)
\right).
\label{eq:proof_final_numerator}
\end{align}

We next simplify the denominator of $\mathcal R_{\alpha,\rho}$.
By the definition of R\'enyi divergence,
\begin{align}
D_\rho(q\|r)
=
\frac{1}{\rho-1}
\log
\sum_{a\in\mathcal V}
q(a)^\rho r(a)^{1-\rho}.
\label{eq:proof_reference_renyi_definition}
\end{align}
Multiplying both sides of
\Cref{eq:proof_reference_renyi_definition} by $\rho-1$ gives
\begin{align}
(\rho-1)D_\rho(q\|r)
=
\log
\sum_{a\in\mathcal V}
q(a)^\rho r(a)^{1-\rho}.
\label{eq:proof_reference_renyi_log}
\end{align}
Exponentiating both sides of
\Cref{eq:proof_reference_renyi_log} gives
\begin{align}
\exp\!\left(
(\rho-1)D_\rho(q\|r)
\right)
=
\sum_{a\in\mathcal V}
q(a)^\rho r(a)^{1-\rho}.
\label{eq:proof_final_denominator}
\end{align}

Using \Cref{eq:proof_final_numerator,eq:proof_final_denominator} in the
definition of $\mathcal R_{\alpha,\rho}$ gives
\begin{align}
&\mathcal R_{\alpha,\rho}
(\theta;h,h_{\mathrm{sf}},h_{\mathrm{ref}})
\nonumber\\
&=
-\frac{1}{\rho-1}
\log
\frac{
Z_\alpha^\rho
\exp\!\left(
(\rho-1)
D_\rho\!\left(
\pi_\alpha^\star(\cdot\mid h_{\mathrm{sf}},h_{\mathrm{ref}})
\,\middle\|\,
q
\right)
\right)
}{
\exp\!\left(
(\rho-1)D_\rho(q\|r)
\right)
}.
\label{eq:proof_substitute_power_sums}
\end{align}
We then obtain
\begin{align}
&\mathcal R_{\alpha,\rho}
(\theta;h,h_{\mathrm{sf}},h_{\mathrm{ref}})
\nonumber\\
&=
-\frac{1}{\rho-1}
\Bigg[
\log\!\left(
Z_\alpha^\rho
\exp\!\left(
(\rho-1)
D_\rho\!\left(
\pi_\alpha^\star(\cdot\mid h_{\mathrm{sf}},h_{\mathrm{ref}})
\,\middle\|\,
q
\right)
\right)
\right)
\nonumber\\
&\hspace{4.8cm}
-
\log\!\left(
\exp\!\left(
(\rho-1)D_\rho(q\|r)
\right)
\right)
\Bigg].
\label{eq:proof_expand_log_ratio}
\end{align}
Expanding the logarithms and suppressing the fixed context arguments gives
\begin{align*}
\mathcal R_{\alpha,\rho}(\theta)
&=-\frac{\rho}{\rho-1}\log Z_\alpha
  -D_\rho(\pi_\alpha^\star\|q)+D_\rho(q\|r).
\end{align*}
Subtracting the reference penalty therefore yields
\begin{align*}
\mathcal J_{\alpha,\rho}(\theta)
&=\mathcal R_{\alpha,\rho}(\theta)-D_\rho(q\|r)\\
&=-D_\rho(\pi_\alpha^\star\|q)
  -\frac{\rho}{\rho-1}\log Z_\alpha.
\end{align*}
For fixed $h$, $h_{\mathrm{sf}}$, and $h_{\mathrm{ref}}$, the
distributions $p$ and $r$ are frozen during the inner student update.
Thus,
\[
Z_\alpha(h_{\mathrm{sf}},h_{\mathrm{ref}})
=\sum_{a\in\mathcal V}p(a)^\alpha r(a)^{1-\alpha}
\]
is independent of $\theta$, so the normalizer term is constant.
Restoring $q=\pi_\theta(\cdot\mid h)$, we obtain
\[
\argmax_\theta\,\mathcal J_{\alpha,\rho}(\theta)
=
\argmin_\theta\,D_\rho\!\left(
\pi_\alpha^\star(\cdot\mid h_{\mathrm{sf}},h_{\mathrm{ref}})
\,\middle\|\,\pi_\theta(\cdot\mid h)\right).
\]
This proves the stated conditional variational equivalence.

When $\alpha=1$, $Z_1=\sum_a p(a)=1$ and $\pi_1^\star=p$. Hence,
\begin{align*}
\mathcal J_{1,\rho}(\theta;h,h_{\mathrm{sf}},h_{\mathrm{ref}})
&=-D_\rho\!\left(
\operatorname{sg}\!\left[\pi_{\bar\theta}(\cdot\mid h_{\mathrm{sf}})\right]
\,\middle\|\,\pi_\theta(\cdot\mid h)\right).
\end{align*}
\end{proof}

\subsection{Proof of Proposition~\ref{prop:renyi_logit_gradient}}
\begin{proof}
By the definition of R\'enyi divergence,
\begin{align}
&D_\rho\!\left(
\pi_\alpha^\star(\cdot\mid h_{\mathrm{sf}},h_{\mathrm{ref}})
\,\middle\|\,
\pi_\theta(\cdot\mid h)
\right)
\nonumber\\
&=
\frac{1}{\rho-1}
\log
\sum_{b\in\mathcal V}
\pi_\alpha^\star(b\mid h_{\mathrm{sf}},h_{\mathrm{ref}})^\rho
\pi_\theta(b\mid h)^{1-\rho}.
\label{eq:renyi_gradient_start}
\end{align}
Define
\begin{align}
S_{\alpha,\rho,\theta}
:=
\sum_{b\in\mathcal V}
\pi_\alpha^\star(b\mid h_{\mathrm{sf}},h_{\mathrm{ref}})^\rho
\pi_\theta(b\mid h)^{1-\rho}.
\label{eq:renyi_gradient_S}
\end{align}
Then
\begin{align*}
D_\rho
=
\frac{1}{\rho-1}\log S_{\alpha,\rho,\theta}.
\end{align*}

We first compute the derivative of the student probability with respect to
the student logit. By the softmax definition,
\begin{align*}
\pi_\theta(b\mid h)
=
\frac{
\exp(z_\theta(b\mid h))
}{
\sum_{c\in\mathcal V}\exp(z_\theta(c\mid h))
}.
\end{align*}
Therefore,
\begin{align}
\frac{\partial\pi_\theta(b\mid h)}
{\partial z_\theta(a\mid h)}
&=
\frac{
\mathbf 1\{a=b\}\exp(z_\theta(b\mid h))
}{
\sum_c\exp(z_\theta(c\mid h))
}
\nonumber\\
&\quad-
\frac{
\exp(z_\theta(b\mid h))
\exp(z_\theta(a\mid h))
}{
\left(\sum_c\exp(z_\theta(c\mid h))\right)^2
}
\nonumber\\
&=
\pi_\theta(b\mid h)
\left[
\mathbf 1\{a=b\}
-
\pi_\theta(a\mid h)
\right].
\label{eq:softmax_derivative}
\end{align}

Since $\pi_\alpha^\star$ is frozen during the student update,
only $\pi_\theta$ depends on $z_\theta$. Hence,
\begin{align*}
\frac{\partial S_{\alpha,\rho,\theta}}
{\partial z_\theta(a\mid h)}
&=
\sum_{b\in\mathcal V}
\pi_\alpha^\star(b\mid h_{\mathrm{sf}},h_{\mathrm{ref}})^\rho
\frac{\partial
\pi_\theta(b\mid h)^{1-\rho}}
{\partial z_\theta(a\mid h)}
\nonumber\\
&=
(1-\rho)
\sum_{b\in\mathcal V}
\pi_\alpha^\star(b\mid h_{\mathrm{sf}},h_{\mathrm{ref}})^\rho
\pi_\theta(b\mid h)^{-\rho}
\frac{\partial\pi_\theta(b\mid h)}
{\partial z_\theta(a\mid h)}.
\end{align*}
Substituting \Cref{eq:softmax_derivative} gives
\begin{align*}
\frac{\partial S_{\alpha,\rho,\theta}}
{\partial z_\theta(a\mid h)}
&=
(1-\rho)
\sum_{b\in\mathcal V}
\pi_\alpha^\star(b\mid h_{\mathrm{sf}},h_{\mathrm{ref}})^\rho
\pi_\theta(b\mid h)^{1-\rho}
\nonumber\\
&\qquad\times
\left[
\mathbf 1\{a=b\}
-
\pi_\theta(a\mid h)
\right].
\end{align*}
Separating the two terms,
\begin{align}
\frac{\partial S_{\alpha,\rho,\theta}}
{\partial z_\theta(a\mid h)}
&=
(1-\rho)
\pi_\alpha^\star(a\mid h_{\mathrm{sf}},h_{\mathrm{ref}})^\rho
\pi_\theta(a\mid h)^{1-\rho}
\nonumber\\
&\quad-
(1-\rho)\pi_\theta(a\mid h)
\sum_{b\in\mathcal V}
\pi_\alpha^\star(b\mid h_{\mathrm{sf}},h_{\mathrm{ref}})^\rho
\pi_\theta(b\mid h)^{1-\rho}
\nonumber\\
&=
(1-\rho)
\left[
\pi_\alpha^\star(a\mid h_{\mathrm{sf}},h_{\mathrm{ref}})^\rho
\pi_\theta(a\mid h)^{1-\rho}
-
\pi_\theta(a\mid h)S_{\alpha,\rho,\theta}
\right].
\label{eq:derivative_S}
\end{align}

Next, differentiating the R\'enyi divergence gives
\begin{align*}
\frac{\partial D_\rho}
{\partial z_\theta(a\mid h)}
&=
\frac{1}{\rho-1}
\frac{1}{S_{\alpha,\rho,\theta}}
\frac{\partial S_{\alpha,\rho,\theta}}
{\partial z_\theta(a\mid h)}.
\end{align*}
Substituting \Cref{eq:derivative_S},
\begin{align*}
\frac{\partial D_\rho}
{\partial z_\theta(a\mid h)}
&=
\frac{1-\rho}{\rho-1}
\left[
\frac{
\pi_\alpha^\star(a\mid h_{\mathrm{sf}},h_{\mathrm{ref}})^\rho
\pi_\theta(a\mid h)^{1-\rho}
}{
S_{\alpha,\rho,\theta}
}
-
\pi_\theta(a\mid h)
\right].
\end{align*}
Since
\[
\frac{1-\rho}{\rho-1}=-1,
\]
we obtain
\begin{align*}
\frac{\partial D_\rho}
{\partial z_\theta(a\mid h)}
&=
\pi_\theta(a\mid h)
-
\frac{
\pi_\alpha^\star(a\mid h_{\mathrm{sf}},h_{\mathrm{ref}})^\rho
\pi_\theta(a\mid h)^{1-\rho}
}{
S_{\alpha,\rho,\theta}
}
\nonumber\\
&=
\pi_\theta(a\mid h)
-
\widetilde{\pi}_{\alpha,\rho,\theta}
(a\mid h,h_{\mathrm{sf}},h_{\mathrm{ref}}) .
\end{align*}

Finally, by definition,
\begin{align*}
\pi_\alpha^\star(a\mid h_{\mathrm{sf}},h_{\mathrm{ref}})
=
\frac{
\operatorname{sg}\!\left[
\pi_{\bar\theta}(a\mid h_{\mathrm{sf}})
\right]^\alpha
\pi_{\mathrm{ref}}(a\mid h_{\mathrm{ref}})^{1-\alpha}
}{
Z_\alpha(h_{\mathrm{sf}},h_{\mathrm{ref}})
}.
\end{align*}
Raising both sides to the power $\rho$ gives
\begin{align*}
\pi_\alpha^\star(a\mid h_{\mathrm{sf}},h_{\mathrm{ref}})^\rho
=
\frac{
\operatorname{sg}\!\left[
\pi_{\bar\theta}(a\mid h_{\mathrm{sf}})
\right]^{\alpha\rho}
\pi_{\mathrm{ref}}(a\mid h_{\mathrm{ref}})^{(1-\alpha)\rho}
}{
Z_\alpha(h_{\mathrm{sf}},h_{\mathrm{ref}})^\rho
}.
\end{align*}
Then, the common factor
$Z_\alpha(h_{\mathrm{sf}},h_{\mathrm{ref}})^{-\rho}$ appears in both the
numerator and denominator and therefore cancels. Hence,
\begin{align*}
\widetilde{\pi}_{\alpha,\rho,\theta}
(a\mid h,h_{\mathrm{sf}},h_{\mathrm{ref}})
\propto
&
\operatorname{sg}\!\left[
\pi_{\bar\theta}(a\mid h_{\mathrm{sf}})
\right]^{\alpha\rho}
\nonumber\\
&\times
\pi_{\mathrm{ref}}(a\mid h_{\mathrm{ref}})^{(1-\alpha)\rho}
\pi_\theta(a\mid h)^{1-\rho}.
\end{align*}
\end{proof}

\clearpage

\section{Qualitative Comparison of Generations}
\label{app:qualitative-generation}

The examples below compare selected Qwen3-8B responses from a synthetic stress test, MMLU, and SciKnowEval. For each method, we present excerpts from Sample 0, preserving the original wording while removing formatting tags and bold emphasis. Omissions are marked by [...], and final answers are listed separately. Red text highlights errors or contradictions. Alongside the responses, we provide brief quality assessments and discuss what each example reveals.

\begingroup
\raggedbottom

\begin{caseeditorial}

\caseheading{1}{Correct calculation, wrong option}{Math}

\begin{caseprompt}\qkey{Question.} An item costs \$89. It is discounted by 30\%, then a 12\% sales tax is applied to the discounted price. What is the final price?\par\par A. \$99.68\quad B. \$69.78\quad C. \$62.3\quad D. \$71\par\smallskip\qkey{Gold.} B\qquad\qkey{Source.} Synthetic stress test\end{caseprompt}

\begin{casecompare}
\textbf{Method} & \textbf{Model response (excerpt)} & \textbf{Quality} \\

\casegood{SR-OPSD (w. ref.)}{
 \qomit{} A 30\% discount reduces the price to \$62.30. A 12\% sales tax is then applied to this discounted price, which results in a final price of \$62.30 + \$7.48 = \$69.78. \qomit{}\par
Therefore, the correct solution is B, which corresponds to the final price of \$69.78.
\answerline{B}}{Consistent}

\casewarn{SR-OPSD (w/o ref.)}{
 \qomit{} - Final price after tax = \$62.30 + \$7.476 = \$69.776, or approximately \$69.78 \qomit{}\par
- Option B is \qerror{not sufficiently aligned} with the tax calculation. \qomit{}\par
- \qerror{Only Option D} accounts for both the discount and the tax, making it the most scientific and logical solution. \qomit{}
\answerline{\qerror{D}}}{Wrong option}

\casefail{SDPO (RKL)}{
The original price of the item is \$89. When a 30\% discount is applied, the price of the item is reduced by 30\%. Therefore, the discounted price is \$89 * (1 - 0.30) = \$62.3. Then, a 12\% sales tax is applied to the discounted price. Therefore, the final price is \$62.3 * (1 + 0.12) = \$69.78.
\answerline{\qerror{D}}}{Wrong option}

\end{casecompare}

\begin{casefinding}
All three responses obtain approximately \$69.78, but only SR-OPSD with reference selects the corresponding option B. The no-reference response explicitly rejects B after computing that value; RKL also selects D. In both incorrect responses, the final answer does not follow from the calculation.
\end{casefinding}

\end{caseeditorial}

\begin{caseeditorial}

\caseheading{2}{Conflicting conclusions}{Physics}

\begin{caseprompt}\qkey{Question.} A certain element emits 1 alpha particle, and its products then emit 2 beta particles in succession. The atomic number of the resulting element is changed by\par\par A. zero\quad B. minus 1\quad C. minus 2\quad D. plus 1\par\smallskip\qkey{Gold.} A\qquad\qkey{Source.} MMLU / Conceptual Physics\end{caseprompt}

\begin{casecompare}
\textbf{Method} & \textbf{Model response (excerpt)} & \textbf{Quality} \\

\casegood{SR-OPSD (w. ref.)}{
 \qomit{} When we combine these two effects, the net change in the atomic number is:  
-2 (from the alpha particle) + 2 (from the beta particles) = 0.  \par
Thus, the atomic number of the resulting element is unchanged, and the correct answer is A.
\answerline{A}}{Consistent}

\casewarn{SR-OPSD (w/o ref.)}{
 \qomit{} - Option D (plus 1) reflects a gain in atomic number, which is \qerror{not consistent} with the actual sequence of events. \qomit{}\par
\qerror{Only Option D correctly captures} the net effect of the beta decay as a gain in atomic number, which is the only scientifically valid outcome in this context. \qomit{}
\answerline{\qerror{D}}}{Contradiction}

\casefail{SDPO (RKL)}{
The original element emits 1 alpha particle. When an alpha particle is emitted, the atomic number of the element is \qerror{decreased by 1}. The products then emit 2 beta particles in succession. When a beta particle is emitted, the atomic number of the element is increased by 1. Therefore, the atomic number of the resulting element is changed by a decrease of 1 and an increase of 1. \qerror{The net change is a decrease of 1.}
\answerline{\qerror{D}}}{Incorrect rule}

\end{casecompare}

\begin{casefinding}
The no-reference response first rejects D, then endorses it. RKL states an incorrect alpha-decay rule and concludes a decrease of 1, yet selects D (plus 1). SR-OPSD obtains zero and selects A.
\end{casefinding}

\end{caseeditorial}

\newpage

\begin{caseeditorial}

\caseheading{3}{Contradictory answer selection}{Anatomy}

\begin{caseprompt}\qkey{Question.} In the spinal cord, motor neuron cell bodies are located in\par\par A. Lamina I.\quad B. Lamina II.\quad C. Lamina V.\quad D. Lamina IX\par\smallskip\qkey{Gold.} D\qquad\qkey{Source.} MMLU / Anatomy\end{caseprompt}

\begin{casecompare}
\textbf{Method} & \textbf{Model response (excerpt)} & \textbf{Quality} \\

\casegood{SR-OPSD (w. ref.)}{
The spinal cord is organized into distinct layers, with motor neuron cell bodies located in lamina IX. This region is known as the ventral horn and is responsible for controlling voluntary movements. \qomit{}
\answerline{D}}{Correct location}

\casewarn{SR-OPSD (w/o ref.)}{
 \qomit{} - Lamina IX, however, is \qerror{not a part of the spinal cord’s functional motor system}, as it lies outside the typical range of motor neuron localization.\par
Given this, the \qerror{correct solution must be Option A}, Lamina I, which \qerror{does not accurately reflect the known anatomical distribution} of motor neurons. \qomit{}
\answerline{\qerror{A}}}{Contradiction}

\casefail{SDPO (RKL)}{
The spinal cord is divided into several laminae. Motor neuron cell bodies are located in the laminae. The laminae are divided into different layers. The motor neuron cell bodies are \qerror{located in the lamina II}.
\answerline{\qerror{B}}}{Factual error}

\end{casecompare}

\begin{casefinding}
The no-reference response calls A correct while stating that it does not match the anatomical distribution. RKL repeats general statements about laminae and identifies the wrong location. SR-OPSD identifies lamina IX and selects its matching option.
\end{casefinding}

\end{caseeditorial}

\begin{caseeditorial}

\caseheading{4}{Repeated reasoning, wrong answer}{Nutrition}

\begin{caseprompt}\qkey{Question.} How many kcal are there in one gram of ethanol?\par\par A. 23.5 kJ or 5.8 kcal per g\quad B. 19.7 kJ or 4.1 kcal per g\quad C. 29.7 kJ or 7.1 kcal per g\quad D. 36.5 kJ or 8.1 kcal per g\par\smallskip\qkey{Gold.} C\qquad\qkey{Source.} MMLU / Nutrition\end{caseprompt}

\begin{casecompare}
\textbf{Method} & \textbf{Model response (excerpt)} & \textbf{Quality} \\

\casegood{SR-OPSD (w. ref.)}{
 \qomit{} Among the given choices, Option C (29.7 kJ or 7.1 kcal per g) is the only one that accurately reflects the energy content of ethanol. It aligns with the known caloric value and is the most scientifically accurate answer.
\answerline{C}}{Consistent}

\casewarn{SR-OPSD (w/o ref.)}{
 \qomit{} This makes \qerror{Option A}, which suggests 23.5 kJ or 5.8 kcal per gram, the most scientifically and realistically accurate choice. \qomit{}\par
- Option C, in contrast, \qerror{overestimates} the energy output, making it less accurate in this context. \qomit{}
\answerline{\qerror{A}}}{Wrong value}

\casefail{SDPO (RKL)}{
 \qomit{} The heat of combustion of ethanol is approximately 29.7 kJ per gram. Therefore, the number of kcal in one gram of ethanol is determined by the heat of combustion. The heat of combustion of ethanol is approximately 29.7 kJ per gram. Therefore, the number of kcal in one gram of ethanol is determined by the heat of combustion.
\answerline{\qerror{B}}}{Repetition}

\end{casecompare}

\begin{casefinding}
RKL repeats the same two sentences about heat of combustion but selects B instead of the matching option C. The no-reference response endorses the wrong value in A. SR-OPSD links the stated energy value to C, whereas the repeated RKL reasoning adds no support for its final choice.
\end{casefinding}

\end{caseeditorial}

\newpage

\begin{caseeditorial}

\caseheading{5}{Answer selection without a numerical derivation}{Chemistry}

\begin{caseprompt}\qkey{Question.} What is the molar weight (g/mol) of the molecule with the the IUPAC name '(2S)-2-[[(2S)-2-[[(2S)-2-[[(2S)-5-amino-2-[[(2S)-2-[[(2S)-2-[[(2S)-1-[(2S)-2-[[(2S)-2-[[(2S)-2-[[(2S)-1-[(2S)-2-amino-3-methylbutanoyl]pyrrolidine-2-carbonyl]amino]-3-(1H-indol-3-yl)propanoyl]amino]-4-methylsulfanylbutanoyl]amino]-4-carboxybutanoyl]pyrrolidine-2-carbonyl]amino]propanoyl]amino]-3-(4-hydroxyphenyl)propanoyl]amino]-5-oxopentanoyl]amino]-5-(diaminomethylideneamino)pentanoyl]amino]-3-phenylpropanoyl]amino]-4-methylpentanoic acid'?\par\par A: 1517.700\quad B: 1589.900\quad C: 1536.800\quad D: 1490.700\par Please reason step by step.\par\smallskip\qkey{Gold.} C\qquad\qkey{Source.} SciKnowEval / Chemistry\end{caseprompt}

\begin{casecompare}
\textbf{Method} & \textbf{Model response (excerpt)} & \textbf{Quality} \\

\casegood{SR-OPSD (w. ref.)}{
 \qomit{} The solution correctly concludes that Option C (1536.800) is the most accurate representation of the molecule's molar weight, based on its composition and the absence of any indication of a significantly different molecular structure. \qomit{}
\answerline{C}}{Correct choice}

\casewarn{SR-OPSD (w/o ref.)}{
 \qomit{} - \qerror{Option A (1517.700)}, with its lower and more moderate weight, is the most accurate solution in line with the molecule's actual composition.\par
- Option C (1536.800), \qerror{while not incorrect, is too high} in its estimation of the molecule's mass. \qomit{}
\answerline{\qerror{A}}}{Unsupported choice}

\casefail{SDPO (RKL)}{
The molar weight of a molecule is determined by summing the atomic weights of all the atoms in the molecule. The given IUPAC name indicates that the molecule contains several hydrophilic and hydrophobic groups. These groups are consistent with a hydrophilic molecule. Therefore, the molar weight of the molecule is determined by summing the atomic weights of all the atoms in the molecule. The hydrophilic groups are more consistent with a hydrophilic molecule. Therefore, the molar weight of the molecule is determined by summing the atomic weights of all the atoms in the molecule.
\answerline{\qerror{D}}}{Off-topic repetition}

\end{casecompare}

\begin{casefinding}
SR-OPSD selects C, matching the answer key, but does not show the atom counts or mass calculation. The no-reference response prefers a lower mass without calculating it. RKL repeats a general calculation rule and discusses hydrophilicity without deriving the requested mass.
\end{casefinding}

\end{caseeditorial}

\endgroup

\clearpage
\section{Extra Experimental Details}

\subsection{Extra Results for Science Q\&A}

\begin{figure}[htp]
    \centering
    \includegraphics[width=1\linewidth]{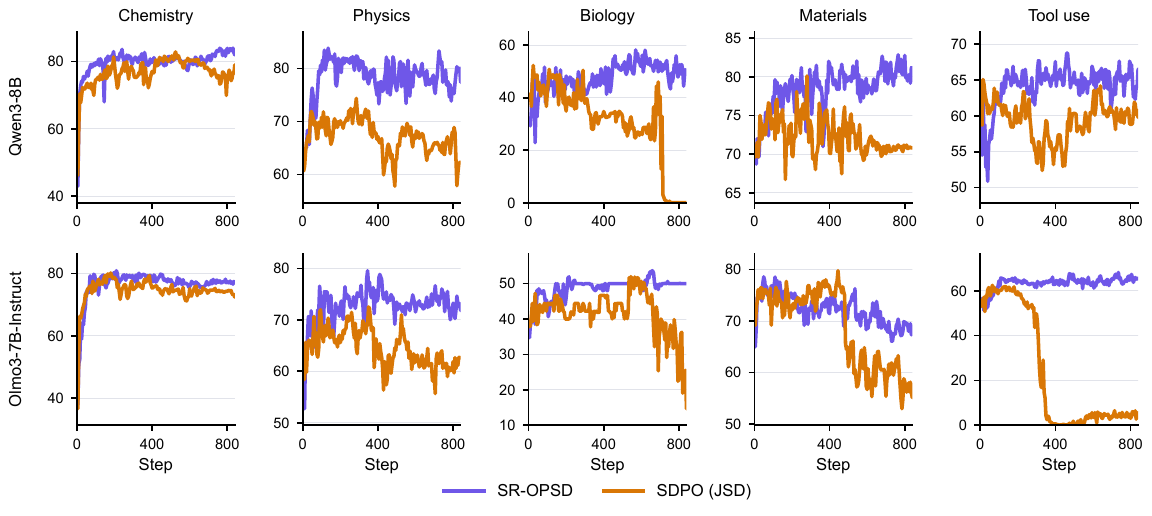}
    \caption{
\textbf{Maj@16 test accuracy over training.}
We report Maj@16 test accuracy across training steps on five reasoning-related benchmarks for Qwen3-8B and Olmo3-7B-Instruct.
}
    \label{fig: major acc on science}
\end{figure}

\begin{figure}[htp]
    \centering
    \includegraphics[width=1\linewidth]{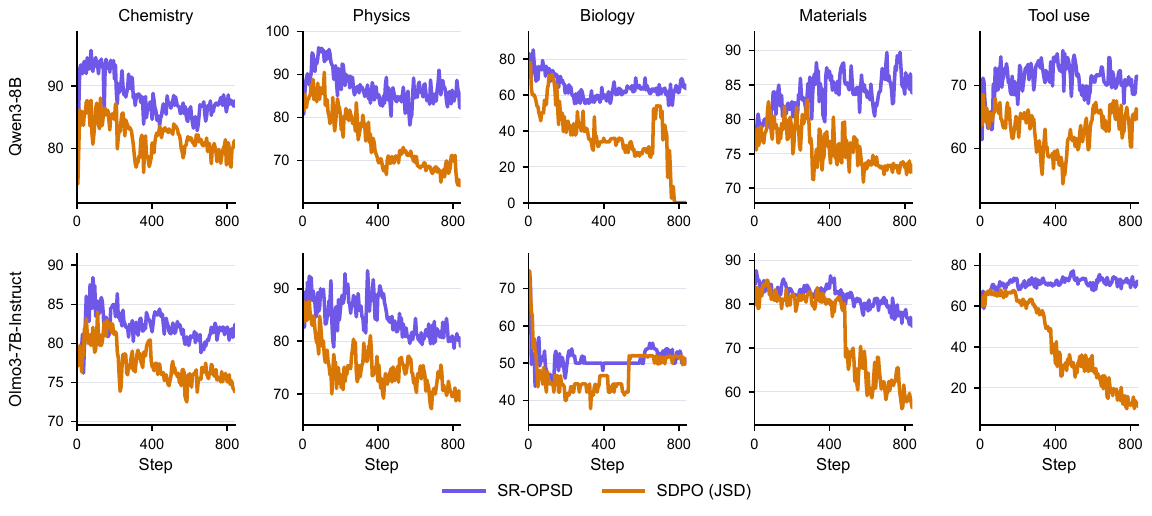}
    \caption{
\textbf{Best@16 test accuracy over training.}
We report Best@16 test accuracy across training steps on five reasoning-related benchmarks for Qwen3-8B and Olmo3-7B-Instruct.
}
    \label{fig: best acc on science}
\end{figure}

\clearpage
\subsection{Experimental Configurations}
\label{app:implementation_details}

\begin{table*}[htp]
\centering
\scriptsize
\caption{Hyperparameters for the SciKnowEval Science Q\&A and Tool-use experiments.}
\label{tab:sciknoweval_experiment_config}
\resizebox{\textwidth}{!}{%
\begin{tabular}{lccc}
\toprule
\textbf{Parameter} & \textbf{GRPO} & \textbf{SDPO (JSD)} & \textbf{SR-OPSD (Ours)} \\
\midrule
\multicolumn{4}{l}{\textbf{General}} \\
Benchmark & SciKnowEval & SciKnowEval & SciKnowEval \\
Models & \multicolumn{3}{c}{Qwen3-8B and OLMo-3-7B-Instruct} \\
Primary evaluation metrics & Avg@16 & Avg@16 & Avg@16 \\
Evaluation thinking mode & Disabled & Disabled & Disabled \\
\midrule
\multicolumn{4}{l}{\textbf{Data, rollout, and evaluation}} \\
Maximum prompt length & 2048 & 2048 & 2048 \\
Maximum training response length & 8192 & 8192 & 8192 \\
Question batch / PPO mini-batch size & 32 / 8 & 32 / 32 & 32 / 32 \\
Per-device batch / gradient accumulation & 1 / 16 & 1 / 64 & 1 / 64 \\
Training rollouts per question & 8 & 8 & 8 \\
Training temperature / top-$p$ / top-$k$
& 1.0 / 1.0 / $-1$ & 1.0 / 1.0 / $-1$ & 1.0 / 1.0 / $-1$ \\
Evaluation samples per question & 16 & 16 & 16 \\
Evaluation temperature / top-$p$ / top-$k$
& 0.6 / 0.95 / $-1$ & 0.6 / 0.95 / $-1$ & 0.6 / 0.95 / $-1$ \\
\midrule
\multicolumn{4}{l}{\textbf{Method-specific objective}} \\
Policy/distillation objective & GRPO & Jensen--Shannon & Forward R\'enyi \\
PPO clipping $\epsilon_{\mathrm{low}}$ / $\epsilon_{\mathrm{high}}$
& 0.20 / 0.28 & -- & -- \\
R\'enyi order $\rho$ & -- & -- & 0.9/0.95 \\
Self-reference coefficient $\alpha$ & -- & -- & 0.7/0.90 \\
Frozen-reference anchoring & No & No & Yes \\
Reference policy & -- & -- & Frozen initial policy \\
Top-$K$ distillation / tail bucket & -- & 100 / Yes & 100 / Yes \\
Teacher EMA update rate & -- & 0.05 & 0.05 \\
Distillation IS clip & -- & 2.0 & 2.0 \\
Token-loss clip & -- & -- & -- \\
\midrule
\multicolumn{4}{l}{\textbf{Optimization}} \\
Optimizer & AdamW & AdamW & AdamW \\
Learning rate & $5\times10^{-6}$ & $1\times10^{-5}$ & $1\times10^{-5}$ \\
Learning-rate schedule / warmup steps
& constant / 10 & constant / 10 & constant / 10 \\
Weight decay / gradient clipping norm
& 0.01 / 1.0 & 0.01 / 1.0 & 0.01 / 1.0 \\
\bottomrule
\end{tabular}%
}
\end{table*}

\begin{table*}[htp]
\centering
\small
\caption{Hyperparameters for the Qwen3-4B and Qwen3-8B mathematical-reasoning experiments.}
\label{tab:math_experiment_config}
\resizebox{\textwidth}{!}{%
\begin{tabular}{lcccc}
\toprule
\textbf{Parameter} & \textbf{GRPO} & \textbf{SDPO (RKL)} & \textbf{OPSD (FKL)} & \textbf{SR-OPSD (Ours)} \\
\midrule
\multicolumn{5}{l}{\textbf{General}} \\
Model & \multicolumn{4}{c}{Qwen3-4B and Qwen3-8B} \\
Benchmark & \multicolumn{4}{c}{AIME24, AIME25, HMMT25, AMC23, Minerva} \\
Primary evaluation metrics & \multicolumn{4}{c}{Avg@16 and Pass@16} \\
Evaluation thinking mode & Enabled & Enabled & Enabled & Enabled \\
\midrule
\multicolumn{5}{l}{\textbf{Data, rollout, and evaluation}} \\
Maximum training response length & 16384 & 16384 & 16384 & 16384 \\
Question batch size & 8 & 8 & 64 & 8 \\
Training rollouts per question & 8 & 8 & 1 & 8 \\
Training rollout thinking mode & Disabled & Disabled & Disabled & Disabled \\
Training temperature / top-$p$ / top-$k$ & \multicolumn{4}{c}{0.7 / 0.95 / 20} \\
Evaluation temperature / top-$p$ / top-$k$ (4B) & \multicolumn{4}{c}{0.7 / 0.95 / 20} \\
Evaluation temperature / top-$p$ / top-$k$ (8B) & \multicolumn{4}{c}{1.0 / 1.0 / $-1$} \\
\midrule
\multicolumn{5}{l}{\textbf{Method-specific objective}} \\
Policy/distillation objective & GRPO & Reverse KL & Forward KL & Forward R\'enyi \\
R\'enyi order $\rho$ (4B / 8B) & -- & -- & -- &0.7 / 0.9 \\
Self-reference coefficient $\alpha$ (4B / 8B) & -- & -- & -- & 0.9 / 0.9\\
Frozen-reference anchoring & No & No & No & Yes \\
Reference policy & -- & -- & -- & Frozen initial policy \\
Top-$K$ distillation / tail bucket & -- & 100 / No & 100 / No & 100 / No \\
Teacher EMA update rate & -- & 0.05 & -- & 0.05 \\
Token-loss clip (4B / 8B) & -- & 0.05 / 0.05 & 0.05 / 0.06 & 0.05 / 0.05 \\
GRPO policy-ratio clip & 0.2 & -- & -- & -- \\
\midrule
\multicolumn{5}{l}{\textbf{Optimization}} \\
Optimizer & AdamW & AdamW & AdamW & AdamW \\
Learning rate & \multicolumn{4}{c}{$5\times10^{-6}$} \\
Learning-rate schedule / warmup steps & \multicolumn{4}{c}{linear / 0} \\
Weight decay / gradient clipping norm & \multicolumn{4}{c}{0 / 0.1} \\
\bottomrule
\end{tabular}%
}
\end{table*}

\begin{table*}[htp]
\centering
\small
\caption{Hyperparameters for the Qwen3 model-scaling experiments on LiveCodeBench v6. Training uses eight rollouts per prompt, whereas test uses 16 sampled solutions per problem. }
\label{tab: experiments config for LCB task}
\resizebox{0.92\textwidth}{!}{%
\begin{tabular}{lccc}
\toprule
\textbf{Parameter} & \textbf{GRPO} & \textbf{SDPO} & \textbf{SR-OPSD (Ours)} \\
\midrule

\multicolumn{4}{l}{\textbf{General}} \\
Models & \multicolumn{3}{c}{Qwen3-0.6B, Qwen3-1.7B, Qwen3-4B, Qwen3-8B} \\
Benchmark & \multicolumn{3}{c}{LiveCodeBench v6} \\
Training and evaluation thinking mode & Disabled & Disabled & Disabled \\
\midrule

\multicolumn{4}{l}{\textbf{Data, rollout, and test}} \\
Maximum prompt / response length & 2048 / 8192 & 2048 / 8192 & 2048 / 8192 \\
Question batch / PPO mini-batch size & 32 / 32 & 32 / 32 & 32 / 32 \\
Training rollouts per question & 8 & 8 & 8 \\
Training temperature / top-$p$ / top-$k$ & 1.0 / 1.0 / $-1$ & 1.0 / 1.0 / $-1$ & 1.0 / 1.0 / $-1$ \\
Test frequency / samples per question & 5 / 16 & 5 / 16 & 5 / 16 \\
Test temperature / top-$p$ / top-$k$ & 0.6 / 0.95 / $-1$ & 0.6 / 0.95 / $-1$ & 0.6 / 0.95 / $-1$ \\
\midrule

\multicolumn{4}{l}{\textbf{Method-specific objective}} \\
Policy/distillation objective & GRPO & Reverse KL & Forward R\'enyi \\
Environment feedback in reprompting & -- & Yes & Yes \\
R\'enyi order & -- & -- & 0.95 \\
Self-reference coefficient & -- & -- & 0.90 \\
Frozen-reference anchoring & -- & No & Yes \\
Reference policy & -- & -- & Frozen initial policy \\
Top-$K$ distillation / tail bucket & -- & 100 / Yes & 100 / Yes \\
Teacher EMA update rate & -- & 0.05 & 0.05 \\
Rollout IS correction / distillation IS clip & token / -- & token / 2.0 & token / 2.0 \\
\midrule

\multicolumn{4}{l}{\textbf{Optimization}} \\
Optimizer & AdamW & AdamW & AdamW \\
Learning rate / schedule & $1\times10^{-5}$ / constant & $1\times10^{-5}$ / constant & $1\times10^{-5}$ / constant \\
Warmup steps & 10 & 10 & 10 \\
Weight decay / gradient clipping norm & 0.01 / 1.0 & 0.01 / 1.0 & 0.01 / 1.0 \\
PPO epochs / loss aggregation & 1 / token mean & 1 / token mean & 1 / token mean \\
\bottomrule
\end{tabular}%
}
\end{table*}

\clearpage
\subsection{Extra Details of Ablation Study}
\label{appendix:Extra Details of Ablation Study}

\begin{table}[htp]
\centering
\caption{Ablation study on Science Q\&A Physics. For SR-OPSD, $\alpha$
denotes the self-reference coefficient and $\rho$ denotes the R\'enyi order.
Each row reports the optimal Avg@16 within 10 compute hours.}
\label{tab:ablation_study}
\resizebox{\textwidth}{!}{%
\begin{tabular}{@{}lllccrrrr@{}}
\toprule
\multirow{2}{*}{\textbf{Model}} &
\multirow{2}{*}{\textbf{Method}} &
\multirow{2}{*}{\textbf{Reference}} &
\multicolumn{2}{c}{\textbf{R\'enyi}} &
\multicolumn{4}{c}{\textbf{Physics}} \\
\cmidrule(lr){4-5}
\cmidrule(lr){6-9}
& & & $\boldsymbol{\alpha}$ & $\boldsymbol{\rho}$
& \textbf{Avg@16}
& \textbf{Pass@16}
& \textbf{Maj@16}
& \textbf{Format} \\
\midrule
\multirow{17}{*}{Qwen3-8B}
& JSD
& No
& --
& --
& 73.4
& 88.3
& 81.7
& 99.8 \\
& JSD
& Yes
& 0.9
& --
& 77.2
& 84.6
& 79.4
& 99.8 \\
& FKL
& No
& --
& --
& 79.1
& 91.9
& 81.8
& 100.0 \\
& FKL
& Yes
& 0.9
& --
& 76.3
& 82.3
& 78.3
& 100.0 \\
& Forward R\'enyi
& No
& --
& 0.95
& 77.7
& 92.6
& 81.8
& 100.0 \\
\cmidrule(l){2-9}
& \multirow{12}{*}{SR-OPSD (Ours)}
& \multirow{12}{*}{Yes}
& \multirow{4}{*}{0.50}
& 0.50
& 78.8
& 91.6
& 82.7
& 99.6 \\
& & &
& 0.70
& 78.8
& 92.6
& 81.2
& 99.5 \\
& & &
& 0.90
& 78.2
& 94.5
& 81.9
& 99.7 \\
& & &
& 0.95
& 78.4
& 92.0
& 80.9
& 99.1 \\
\cmidrule(l){4-9}
& &
&
\multirow{4}{*}{0.70}
& 0.50
& 76.3
& 85.5
& 77.6
& 99.9 \\
& & &
& 0.70
& 82.0
& 94.3
& 85.0
& 99.9 \\
& & &
& 0.90
& 77.0
& 93.9
& 78.0
& 99.3 \\
& & &
& 0.95
& 80.7
& 94.2
& 81.6
& 99.5 \\
\cmidrule(l){4-9}
& &
&
\multirow{4}{*}{0.90}
& 0.50
& 79.9
& 88.9
& 80.0
& 100.0 \\
& & &
& 0.70
& 82.7
& 91.0
& 84.2
& 100.0 \\
& & &
& 0.90
& 77.2
& 90.4
& 79.8
& 99.9 \\
& & &
& 0.95
& 81.2
& 86.8
& 82.4
& 100.0 \\
\bottomrule
\end{tabular}%
}
\end{table}

On SciKnowEval Physics, SR-OPSD maintains relatively stable response lengths, whereas JSD and forward KL exhibit pronounced length growth in later training stages. Self-Reference anchoring does not prevent this growth under JSD, while the self-reference anchored forward-KL variant produces even longer responses than its unanchored counterpart. These trajectories complement the accuracy ablation, further indicating that the effect of self-referenced anchoring depends on the projection objective.
\begin{figure}[htb]
    \centering
    \includegraphics[width=0.8\linewidth]{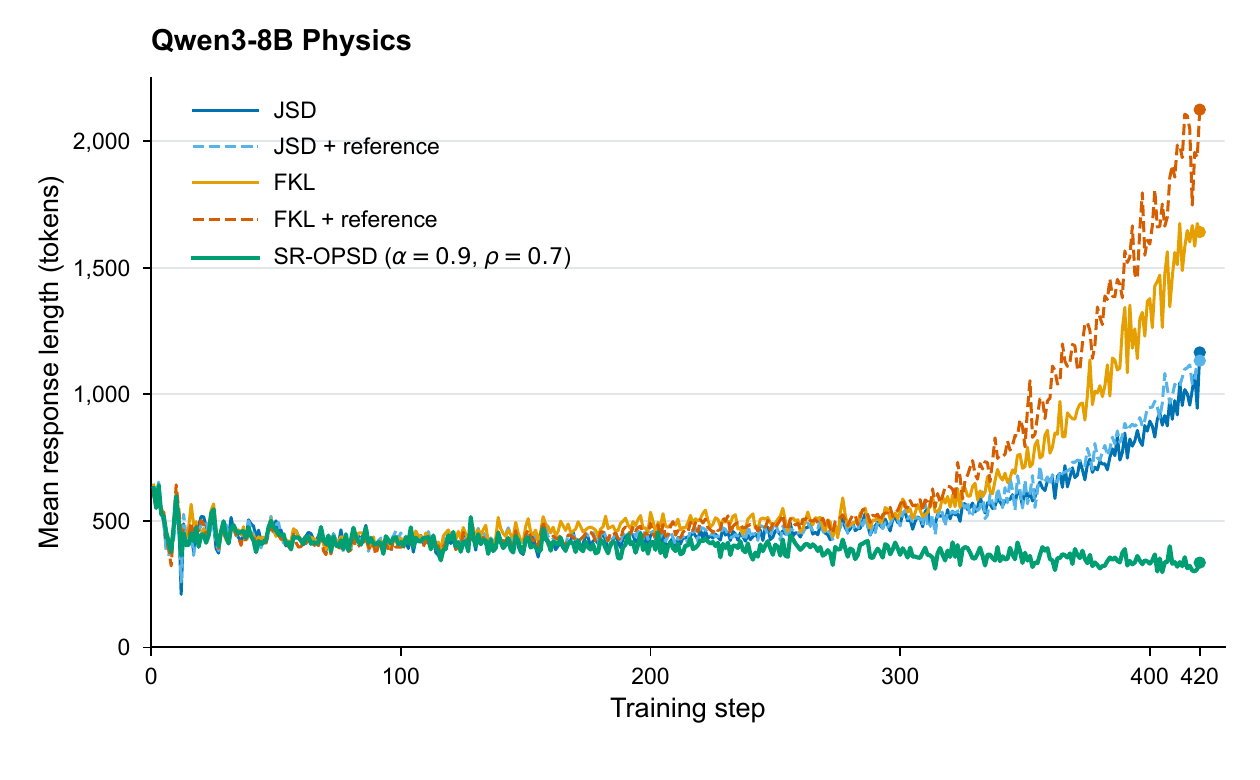}
    \caption{Investigation of response lengths.}
    \label{fig:physics_response_lengths}
\end{figure}

\end{document}